\documentclass{article}
\usepackage[utf8]{inputenc}
\usepackage[margin=0.75in]{geometry}

\usepackage{amsfonts}
\usepackage{amsmath}
\usepackage{amssymb}
\usepackage{amsthm}
\usepackage{mathtools}
\usepackage{nicematrix}

\usepackage{rotating}
\usepackage{graphicx}
\usepackage{subfigure}
\usepackage{color}
\usepackage{tikz-cd}
\usetikzlibrary{arrows,backgrounds}
\usepackage{verbatim}
\usepackage{hyperref}
\usepackage[]{algorithm2e}
\usepackage{bm}

\usepackage{cleveref}

\usepackage{thmtools}
\usepackage{thm-restate}

\usepackage{multirow}

\newtheorem{lemma}{Lemma}

\theoremstyle{definition}
\newtheorem{definition}{Definition}

\theoremstyle{remark}
\newtheorem*{remark}{Remark}

\newcommand{\R}{\mathbb{R}}
\newcommand{\norm}[1]{\left \lVert #1 \right \rVert}

\newcommand{\argmin}{\text{arg min}}

\title{Discovery of Probabilistic Dirichlet-to-Neumann Maps on Graphs}
\author{Adrienne M. Propp, Jonas A. Actor, Elise Walker, \\ Houman Owhadi,  Nathaniel Trask, Daniel M. Tartakovsky}

\begin{document}
\maketitle

\begin{abstract}
Dirichlet-to-Neumann maps enable the coupling of multiphysics simulations across computational subdomains by ensuring continuity of state variables and fluxes at artificial interfaces. We present a novel method for learning Dirichlet-to-Neumann maps on graphs using Gaussian processes, specifically for problems where the data obey a conservation constraint from an underlying partial differential equation. Our approach combines discrete exterior calculus and nonlinear optimal recovery  to infer relationships between vertex and edge values. This framework yields data-driven predictions with uncertainty quantification across the entire graph, even when observations are limited to a subset of vertices and edges. By optimizing over the reproducing kernel Hilbert space norm while applying a maximum likelihood estimation penalty on kernel complexity, our method ensures that the resulting surrogate strictly enforces conservation laws without overfitting. We demonstrate our method on two representative applications: subsurface fracture networks and arterial blood flow. Our results show that the method maintains high accuracy and well-calibrated uncertainty estimates even under severe data scarcity, highlighting its potential for scientific applications where limited data and reliable uncertainty quantification are critical.
\end{abstract}

\section{Introduction}
Multiphysics simulations are often composed of independent submodules exchanging information or fluxes through shared interfacial states \cite{dvorak2005unifying}. A modular ``systems-of-systems'' perspective is therefore attractive for the prediction, design, and control of large-scale systems whose sheer complexity renders traditional simulation intractable. For example, the Energy Exascale Earth System Model (E3SM) forecasts planetary climate dynamics by coupling submodels representing the atmosphere, ocean, sea and land ice, and terrestrial systems. These submodels interact by exchanging mechanical and transport fluxes determined by their physical states \cite{leung2020introduction}. Similarly, embedded electronic systems rely on individual circuit components that exchange voltage and current information, enabling the construction of integrated systems made up of billions of transistors.

This exchange of state across subsystem boundaries can be formalized using \textit{Dirichlet-to-Neumann (D2N) maps}. These maps arise in linear elliptic boundary value problems, where solving a partial differential equation (PDE) with given Dirichlet boundary conditions yields corresponding fluxes at those boundaries \cite{natarajan1995domain}. D2N maps thus provide a natural abstraction for characterizing interactions between submodels sharing common interfaces. Many established domain decomposition techniques---including mortar methods \cite{arbogast2007multiscale,bernardi1993domain}, FETI methods \cite{farhat2001feti}, and hybridizable methods \cite{cockburn2009unified}---can be viewed within the D2N framework. The rigorous theoretical foundations underpinning traditional D2N maps offer convergence and stability guarantees when coupling large legacy simulation codes \cite{jiang2024structure,sockwell2020interface}.
However, the evaluation of a D2N map traditionally requires computationally intensive solves of the governing equations within each subdomain, which quickly becomes prohibitively expensive in multi-physics and multiscale settings.

Recent efforts have therefore turned to data-driven schemes that replace costly subdomain models with a learned approximation of the D2N map \cite{bochev2024dynamic,hall2021ginns,jiang2024structure}. While such surrogates dramatically reduce computational cost, they generally lack the theoretical guarantees of traditional D2N methods. The absence of rigorous a-priori error or stability guarantees for data-driven D2N surrogates poses significant challenges for verification, validation, and reliable deployment of learned models.

We address this gap by combining discrete exterior calculus (DEC) \cite{TRASK2022110969} with computational graph completion (CGC) --- a generalization of the optimal recovery problem (ORP) to the nonlinear graph setting \cite{owhadi2022computational}. Our framework learns flux-state relationships on graphs while producing analytically and computationally tractable posterior distributions that characterize uncertainty in the D2N map. Together, CGC and DEC deliver an efficient optimization strategy along with closed-form expressions for prediction error and out-of-distribution inference, re-introducing rigorous guarantees into data-driven D2N coupling.

The remainder of this paper is organized as follows. \Cref{sec:optimal_recovery} reviews the standard ORP formulation and its connection to Gaussian processes; \Cref{sec:DEC} summarizes the DEC framework for enforcing conservation on graphs; \Cref{sec:formulation} derives our coupled graph-based ORP-DEC algorithm and uncertainty quantification approach; and \Cref{sec:results} presents numerical experiments on subsurface fracture networks and arterial flow to demonstrate the application of our approach.

\section{Optimal recovery on graphs via Gaussian processes}\label{sec:optimal_recovery}

Our problem of learning a functional relationship from limited and potentially noisy data fits broadly within the framework of the optimal recovery problem (ORP). In mathematics, ORP involves finding the best approximation of an unknown function from limited observations. Given data generated by an unknown true function $f$, which we assume lies in some space $\mathcal{F}$ endowed with norm $\|\cdot\|$, classical ORP identifies the candidate $\hat{f}\in\mathcal{F}$ that minimizes the worst-case relative error:
\begin{align}
\min_{\hat{f}\in\mathcal{F}}\max_{f\in\mathcal{F}} \frac{\|f-\hat{f}\|^2}{\|f\|^2}.
\end{align}
This topic is well-explored in the literature \cite{donoho_statistical_1994,micchelli_survey_1977, Micchelli1976, Owhadi_Scovel_2019}. Variations of this problem include sensor reconstruction problems (e.g., \cite{reconstruction}) and boundary value inverse problems (e.g.,~\cite{ibvp_2000}). More recently, the computational graph completion (CGC) framework \cite{owhadi2022computational} used Gaussian processes (GPs) to generalize the ORP to the graph setting, where constraints may be nonlinear.

GPs offer a powerful probabilistic framework for approaching the optimal recovery problem. A GP defines a prior distribution over a space of functions that can be updated with observed data to obtain a posterior distribution. The posterior mean provides the optimal recovery of the unknown function given the data, while the posterior covariance quantifies uncertainty. This makes GPs particularly useful for applications where both function approximation and uncertainty quantification are crucial. In this section, we provide a brief overview of Gaussian process regression, but we refer the interested reader to \cite{rasmussen_gaussian_2008} for a more comprehensive but accessible treatment.

A Gaussian process is a collection of random variables such that any finite subset follows a multivariate normal distribution. Concretely, a Gaussian process $f(\mathbf{x})$ is specified by its mean and covariance functions, which we denote by $\mu(\mathbf{x})$ and $K(\mathbf{x},\mathbf{x}'),$ respectively:
\begin{align}
    f(\mathbf{x})&\sim GP \big(\mu(\mathbf{x}),K(\mathbf{x},\mathbf{x}')\big)\\
    \mu(\mathbf{x})&=\mathbb{E}[f(\mathbf{x})]\\
    K(\mathbf{x},\mathbf{x}')&=\mathbb{E}[\big(f(\mathbf{x})-\mu(\mathbf{x})\big)\big(f(\mathbf{x}')-\mu(\mathbf{x}')\big)].
\end{align}
Given observations $(\textbf{X},\textbf{Y})=(x_i,y_i)_{i=1}^{N_{data}}\subset\mathcal{X}\times\mathcal{Y},$ where $f(x_i)=y_i$, the joint distribution of the observed outputs $\textbf{Y}$ and unobserved outputs $f(\textbf{x})$ is:
\begin{align}
\begin{bmatrix}
f(\textbf{x}) \\
\mathbf{Y}
\end{bmatrix}
\sim \mathcal{N}\left(
\mathbf{0},
\begin{bmatrix}
K(\textbf{x},\textbf{x}) & K(\textbf{x}, \textbf{X}) \\
K(\textbf{X}, \textbf{x}) & K(\textbf{X},\textbf{X})
\end{bmatrix}
\right).
\end{align}
Conditioning on the observed data yields the posterior distribution for $f(\textbf{x})$:
\begin{align}
    \begin{split}
    f(\mathbf{x})|\textbf{x},\textbf{X},\textbf{Y}\sim \mathcal{N} \big(&K(\mathbf{x},\mathbf{X})K(\mathbf{X},\mathbf{X})^{-1}\textbf{Y},\\
    &K(\mathbf{x},\mathbf{x})-K(\mathbf{x},\mathbf{X})K(\mathbf{X},\mathbf{X})^{-1}K(\mathbf{X},\mathbf{x})\big).
    \end{split}
\end{align}
If, instead, we assume that our observations contain some noise---that is, that $f(x_i)=y_i+\epsilon_i$ with Gaussian noise $\epsilon\sim\mathcal{N}(0,\sigma_{\epsilon}^2)$---our posterior distribution becomes:
\begin{align}
    \begin{split}
    f(\mathbf{x})|\textbf{x},\textbf{X},\textbf{Y}\sim \mathcal{N} \big(&K(\mathbf{x},\mathbf{X})[K(\mathbf{X},\mathbf{X})+\sigma_{\epsilon}^2I]^{-1}\textbf{Y},\\
    &K(\mathbf{x},\mathbf{x})-K(\mathbf{x},\mathbf{X})[K(\mathbf{X},\mathbf{X})+\sigma_{\epsilon}^2I]^{-1}K(\mathbf{X},\mathbf{x})\big).
    \end{split}
\end{align}
The noisy observation setting is more common in practice and is the one we adopt for the remainder of this work. Note that the posterior covariance consists of a positive term subtracted from the prior covariance $K(\mathbf{x},\mathbf{x})$, representing the reduction in uncertainty gained from the observations.

A GP is characterized by a covariance (kernel) function $K(\cdot,\cdot)$, which must be symmetric and positive-definite. According to the Moore–Aronszajn Theorem, each kernel $K(\cdot,\cdot)$ uniquely induces a corresponding reproducing kernel Hilbert space (RKHS) $\mathcal{H}_{K}$ \cite{aronszajn_theory_1950}. This RKHS contains precisely the functions spanned by the GP kernel, and functions drawn from the GP prior belong to $\mathcal{H}_{K}$ with high probability. Working in $\mathcal{H}_{K}$ therefore yields solutions that are regularized according to the smoothness properties inherited from $K(\cdot,\cdot)$. Thus, prior information about the nature of the true function $f$ can be incorporated into the choice of kernel. The squared exponential RBF kernel:
\begin{equation}
K(x,y) = \exp \left( - \frac{ \norm{x-y}_2^2}{\ell} \right),\label{eq:RBF}
\end{equation}
is a popular choice, where the lengthscale parameter $\ell$ determines the characteristic distance over which the function values remain strongly correlated. This is the kernel we adopt, due to its flexibility and ability to capture a wide range of functional forms.  However, depending on the problem setting, alternatives such as Mat\'ern, linear, Ornstein-Uhlenbeck kernels may also be appropriate. 

In the (nonlinear) optimal recovery setting, restricting our search space to the induced RKHS $\mathcal{H}_{K}$ translates to choosing the minimal-norm solution under kernel-based regularization. Minimizing over the RKHS' accompanying norm $\norm{\cdot}_K$ yields the smoothest function that fits our data according to the smoothness assumptions encoded in $K(\cdot,\cdot)$. Specifically, given data $(\mathbf{X},\mathbf{Y})$ and a kernel-induced RKHS $\mathcal{H}_{K}$ with norm $\norm{\cdot}_K$, we solve the following minimization problem:
\begin{equation}\label{eq:GP_recovery_noiseless}
\begin{split}
	\min_{f \in \mathcal{H}_K}  &\norm{f}_K^2\\
\text{s.t.}\qquad  &f(\mathbf{X})=\mathbf{Y},
 \end{split}
\end{equation}
in the noiseless case, and 
\begin{equation} \label{eq:GP_recovery_noisy}
	\min_{f \in \mathcal{H}_K}  \norm{f}_K^2 +\frac{1}{\sigma_{\epsilon}^2} \norm{ f\left( \mathbf{X} \right) - \mathbf{Y} }_2^2
\end{equation}
in the presence of noise. By the Representer Theorem, the solution to this problem lies in the span of kernel evaluations at the data points, dramatically simplifying the problem \cite{scholkopf_generalized_2001}. Indeed, the GP posterior mean is exactly the minimizer of the RKHS-regularized problem, ensuring both smoothness and consistency with the data.
In \Cref{thm:inner_min}, we state the theorem that yields a closed-form solution for this minimization problem.

\begin{restatable}[]{thm_restate}{representer}
\label{thm:inner_min}
Consider the GP recovery problem
\begin{equation} 
	\ell(\mathbf{X}; \mathbf{Y}) = \min_{f \in \mathcal{H}_K}  \norm{f}_K^2 +\frac{1}{\sigma_{\epsilon}^2} \norm{ f\left( \mathbf{X} \right) - \mathbf{Y} }_2^2.
\end{equation}
Then, the solution to the minimization problem for $\ell(\mathbf{X}; \mathbf{Y})$ is given by
\begin{equation}
\ell(\mathbf{X}; \mathbf{Y}) =  \mathbf{Y}^T \left(K(\mathbf{X},\mathbf{X})+ \sigma_{\epsilon}^2 I\right)^{-1} \mathbf{Y}.
\end{equation}
\end{restatable}
We postpone the proof of \Cref{thm:inner_min} to \Cref{app:RKHS}.

Thus, rather than searching over an entire infinite-dimensional function space, we can simply take a linear combination of the kernel function evaluated at the training points. This significantly reduces computational complexity and enables efficient function recovery even for large-scale problems.

\section{Discrete exterior calculus}\label{sec:DEC}

While GPs provide a robust framework for function approximation and uncertainty quantification, they do not inherently enforce physical constraints such as conservation laws. Many real-world problems are posed on domains with some geometric structure, where function values must satisfy divergence-free conditions, circulation laws, or other physical principles. In our framework, we leverage the discrete exterior calculus (DEC) to overcome this limitation.

DEC extends the framework of differential geometry and exterior calculus to discrete structures like graphs, providing a precise language through which we can define differential operators (e.g., gradient, divergence, and curl) in a discrete setting \cite{desbrun_kanso_tong_2008,hirani_thesis_2003}. It is particularly well-suited for modeling flows on graphs, where quantities such as mass, heat, or electrical charge must obey conservation laws. In prior work, DEC has been used to learn physics from data \cite{TRASK2022110969}. In our framework, we leverage DEC to constrain our graph-based GP recovery problem, ensuring that the learned functions obey conservation principles across the graphs we consider. This integration yields physically meaningful predictions while maintaining the flexibility of GPs.

In this section, we define the graph gradient and graph divergence, which are discrete analogs to their counterparts in the typical continuous setting.  In these definitions and the remainder of this work, we consider a directed graph $G=(\mathcal{V},\mathcal{E})$, where $\mathcal{V}$ denotes the set of vertices and $\mathcal{E}$ denotes the set of edges. We represent edges $e\in\mathcal{E}$ as ordered pairs of vertices, where an edge $e=(v_a,v_b)$ is considered to be directed from source $v_a$ to target $v_b$. We let $\mathbf{F}\in\mathbb{R}^{|\mathcal{E}|}$ be a discrete 1-form, or edge-valued field, whose entry $\mathbf{F}_e$ is the oriented flux along edge $e$. Similarly, we let $\mathbf{u}\in\mathbb{R}^{|\mathcal{V}|}$ be a discrete 0-form, or vertex-valued field, whose entry $\mathbf{u}_v$ is the scalar potential at vertex $v$.

We first define the incidence matrix, which encodes the topology and connectivity of graph $\mathcal{G}$.
\begin{definition}[Incidence matrix]
    Let $\mathcal{G}=(\mathcal{V},\mathcal{E})$ be a graph consisting of vertices $\mathcal{V}$ and edges $\mathcal{E}$. The incidence matrix $\mathbf{D}_0\in\mathbb{R}^{E\times V}$ encodes the relationships between vertices and edges\footnote{We have adopted an edge-vertex orientation for the incidence matrix. This is consistent with the coboundary viewpoint from DEC, where. The opposite convention, $\delta^\intercal_0$, is also correct.}:
    \begin{equation}
    (\mathbf{D}_0)_{e,v} =
    \begin{cases}
        -1 & \text{if $v$ is the source of edge $e$},\\
        1 & \text{if $v$ is the target of edge $e$}, \\
        0 & \text{otherwise}.
    \end{cases}\\
    \end{equation}
\end{definition}
We next define the graph gradient operator $\delta_0$, which quantifies the change in a scalar quantity or 0-form\footnote{In the field of DEC, it is standard to call a scalar quantity that assigns one value to every vertex a discrete 0‑form (or 0‑cochain). Similarly, a quantity defined on oriented edges --- such as a flux or the discrete gradient --- is a discrete 1‑form (1‑cochain).} between the two endpoints of an edge. The operator $\delta_0$ and incidence matrix $\mathbf{D}_0$ are closely related: $\mathbf{D}_0$ is actually the matrix representation of $\delta_0$, capturing the orientation and ordering of the edges and vertices. 

\begin{definition}[Graph gradient operator]
    Let $v_a$ and $v_b$ be vertices in $\mathcal{V}$ connected by edge $e$. Let $\mathbf{u}_\text{e}=(u_a,u_b)$ represent the data living on the vertices of edge $e=(v_a,v_b)$. The graph gradient operator $\delta_0$ maps values on adjacent vertices to a value on their shared edge. In particular:
    \begin{align}
        \delta_0\mathbf{u}_\text{e}=u_b-u_a.
    \end{align} 
    \end{definition}

We next define the graph divergence, which measures the net flow in or out of a vertex via its incident edges.

\begin{definition}[Graph divergence operator]
    Let $v$ be a vertex in $\mathcal{V}$, and let $\mathbf{F}_e$ denote the flow along edge $e$. The graph divergence operator $\delta_0^\intercal$ maps the flow along edges incident to $v$ to a value on $v$ itself. In particular: 
    \begin{align}
    (\delta_0^\intercal \mathbf{F})(v)=\sum_{e=(v,v_i)}\mathbf{F}_e-\sum_{e=(v_j,v)}\mathbf{F}_e.
    \end{align}
\end{definition}

Thus, if $\mathbf{u}\in\mathbb{R}^{V}$ is a vector of 0-forms (that is, values on the vertices of $\mathcal{G}$), then $\delta_0\mathbf{u}\in\mathbb{R}^{E}$ is a vector of 1-forms (consisting of the graph gradient across each edge). Similarly, if $\mathbf{F}\in\mathbb{R}^{E}$ is a vector of 1-forms (values on the edges of $\mathcal{G}$), then $\delta_0^\intercal\mathbf{F}\in\mathbb{R}^{V}$ returns a vector of 0-forms (consisting of the graph divergence on each vertex).

The orientation of edges is arbitrary but must be consistent throughout the computations.

\section{Computation}\label{sec:formulation}
\subsection{Graph setup}
Let $\mathcal{G} = (\mathcal{V},\mathcal{E})$ be a graph, with $V=\lvert \mathcal{V} \rvert$ vertices and $E=\lvert \mathcal{E} \rvert$ edges. Furthermore, assume that observations are availlable only on subsets of $\mathcal{V}$ and $\mathcal{E}$, so that
\begin{equation} \begin{split}
	\mathcal{V} = \mathcal{V}_\text{un} \cup \mathcal{V}_\text{obs} \qquad \text{and} \qquad
	\mathcal{E} = \mathcal{E}_\text{un} \cup \mathcal{E}_\text{obs},
\end{split} \end{equation}
where the subscript $_\text{un}$ marks unobserved vertices and edges, and $_\text{obs}$ marks observed vertices and edges. 
We denote the cardinalities of these sets as
\begin{equation} \begin{split} 
	\lvert \mathcal{V}_\text{un} \rvert = V_\text{un}, \qquad
	\lvert \mathcal{V}_\text{obs} \rvert = V_\text{obs}, \qquad
	\lvert \mathcal{E}_\text{un} \rvert = E_\text{un}, \qquad \text{and} \qquad
	\lvert \mathcal{E}_\text{obs} \rvert &= E_\text{obs}.
\end{split} \end{equation}
For the observed vertices and edges, we assume we have $N_\text{data}$ independent observations, so that we have vertex observations $\mathbf{u}_\text{obs} \in \R^{  V_\text{obs} \times 
 N_\text{data} }$ and edge observations $\mathbf{F}_\text{obs} \in \R^{ E_\text{obs} \times N_\text{data} }$. We aim to solve for the unknown vertex potentials and edge fluxes, $\mathbf{u}_\text{un} \in \R^{V_\text{un}\times  N_\text{data}   }$ and $\mathbf{F}_\text{un} \in \R^{ E_\text{un} \times N_\text{data} }$, respectively. We denote all vertex and edge quantities as
\begin{equation} \begin{split}
	\mathbf{u} = \begin{bmatrix} \mathbf{u}_\text{obs}  \\ \mathbf{u}_\text{un} \end{bmatrix} \in \R^{V \times N_\text{data}   } \qquad \text{ and } \qquad
	\mathbf{F} = \begin{bmatrix} \mathbf{F}_\text{obs} \\ \mathbf{F}_\text{un} \end{bmatrix} \in \R^{E \times N_\text{data} } 
\end{split} \end{equation}

For each edge $e \in \mathcal{E}$, we construct a GP $f_e$ that maps the vertex values at its endpoints to the flux on $e$, \begin{equation}
    f_e(\mathbf{u}_e) \mapsto \mathbf{F}_e,
\end{equation}
where the edge subscript $e$ selects the appropriate rows of $\mathbf{u}$ and $\mathbf{F}$.
We allow the GP length scale parameter to vary between edges, so $f_e$ is equipped with an individual kernel $K_e$ and belongs to its own reproducing-kernel Hilbert space $\mathcal{H}_{K_e}$.

The form of the kernel input---that is, how we define $x$ and $y$ in \eqref{eq:RBF}---is a design choice. A natural option is the pair of endpoint potentials $x=(u_a,u_b)\in\mathbb{R}^2$ for $e=(v_a,v_b)$. In this case, $f_e(\mathbf{u}_e):\mathbb{R}^2\rightarrow\mathbb{R}$. Alternatively, one may supply the discrete edge gradient, defined as the difference between the values at the endpoints: $x=\delta_0 \mathbf{u}_e=u_{v_a}-u_{v_b}\in\mathbb{R}$. In this case, $f_e(\delta_0\mathbf{u}_e):\mathbb{R}\rightarrow\mathbb{R}$. We use the notation $f_e(\mathbf{u}_e)$ generically, where $f_e$ may represent a convolution of functions, and select the encoding that best matches the application at hand. This is discussed more in \Cref{sec:results}. 

Let $\delta_0 \in \R^{E \times V}$ denote the discrete graph gradient operator and $\delta_0^\intercal \in \R^{V \times E}$ denote the graph divergence operator, as defined in \Cref{sec:DEC}. We impose a global divergence-free conservation constraint given by:
\begin{equation}
	\delta_0 ^\intercal \mathbf{F}  = 0. 
\end{equation}
This global conservation constraint ensures that the modeled flows are physically consistent across the entire graph, effectively inducing a coupling across all GPs $f_e$ for all $e \in \mathcal{E}$.

\subsection{Formulation of the optimization problem}
Our goal is to recover the smoothest set of edge-wise functions $f_e$ that (i) honor the smoothness assumptions encoded by our choice of kernel, (ii) fit the noisy observations, and (iii) satisfy the global conservation law $\delta_0^\intercal\mathbf{F}=0$. Merging the standard GP log-marginal likelihood with the divergence-free constraint expressed via DEC yields the constrained optimisation problem:
\begin{equation} \begin{split} 
	\min_{\theta, \mathbf{u}_\text{un}} \min_{\mathbf{F}_\text{un}} \sum_{e \in \mathcal{E}} &\min_{f_e \in \mathcal{H}_{K_e}}   \norm{ f_e }_{K_e}^2 + \frac{1}{\sigma_{\epsilon}^2} \norm{ f_e(\mathbf{u}_e ) - \mathbf{F}_e }_2^2 + \log \det (K_e(\mathbf{u}_e, \mathbf{u}_e ) + \sigma_{\epsilon}^2 I) \\
	&\text{s.t.}\qquad \delta_0 ^\intercal \mathbf{F} = 0,\\
\end{split}\label{prob:gp-recovery} \end{equation}
where $I$ is the $N_\text{data} \times N_\text{data}$ identity matrix, and $\theta$ represents all learnable parameters, including the edge-dependent length scale $\ell_e$ and overall noise level $\sigma_{\epsilon}^2$. The first term promotes smoothness and regularity  in $f_e$ by minimizing its reproducing kernel Hilbert space (RKHS) norm. The second term penalizes deviations from the observed data, scaled by noise variance $\sigma_{\epsilon}^2$, thus accommodating noise in the observations. The third term penalizes more complex models, discouraging overfitting. By the Representer Theorem for RKHS spaces and the result in \Cref{thm:inner_min}, this is equivalent to the problem:
\begin{equation} \begin{split}
	\min_{\theta, \mathbf{u}_\text{un}} \min_{\mathbf{F}_\text{un}}  \sum_{e \in \mathcal{E}} & \mathbf{F}_e^\intercal (K_e(  \mathbf{u}_e,  \mathbf{u}_e) + \sigma_{\epsilon}^2 I)^{-1} \mathbf{F}_e  + \log \det (K_e( \mathbf{u}_e,  \mathbf{u}_e) + \sigma_{\epsilon}^2 I)\\
	&\text{s.t.} \qquad \delta_0 ^\intercal \mathbf{F} = 0.\\
\end{split} \label{prob:gp-representer} \end{equation}
The inner minimization problem, for fixed $\mathbf{u}_\text{un}$ and $\theta$, is a linearly-constrained quadratic program (LCQP) over $\mathbf{F}_\text{un}$. With the aim of solving this LCQP via the Karush-Kuhn-Tucker (KKT) system, we rewrite this inner problem in a block format and introduce the Lagrange multiplier $\lambda$ to enforce the divergence-free constraint. This minimization sub-problem then becomes:
\begin{equation} \begin{split}
	\min_{F; \lambda} F^\intercal \widehat{K} F + \lambda^\intercal (\widehat{D_0}^\intercal F - b)
\end{split} \label{prob:lcqp} \end{equation}
where
\begin{alignat}{2}\label{eq:matrices}
	F &= \text{vec}(\mathbf{F}_\text{un}) &&\in \mathbb{R}^{ E_{\text{un}}\cdot  N_\text{data}} \\
	\widehat{K} &= \text{diag}(K_e( \mathbf{u}_e,  \mathbf{u}_e) + \sigma_{\epsilon}^2 I)_{e\in\mathcal{E}_{\text{un}}}^{-1} &&\in \mathbb{R}^{( E_{\text{un}}\cdot N_{\text{data}}) \times (E_\text{un}\cdot N_\text{data}) }\\
	D_0 &= {\delta_0}_{\mathcal{E}_\text{un}, \mathcal{V}_\text{un}} &&\in \mathbb{R}^{E_{\text{un}}\times V_\text{un}} \\
	\widehat{D_0} &= D_0 \otimes I &&\in \mathbb{R}^{(E_\text{un}\cdot N_\text{data})\times (V_\text{un}\cdot N_\text{data})}\\
	\mathbf{b} &= - {\delta_0}_{\mathcal{E}_\text{obs}, \mathcal{V}_\text{un}}^\intercal \mathbf{F}_\text{obs}  &&\in \mathbb{R}^{ V_\text{un}\times N_\text{data}}\\
	b &= \text{vec}( \mathbf{b} ) &&\in \mathbb{R}^{V_\text{un} \cdot N_\text{data} } 
\end{alignat}
and $\text{vec}(\cdot)$ is the row-major vectorization of a matrix. Detailed specifications of these matrices are provided in \Cref{app:matrices}.

This LCQP is efficiently solved via the KKT system, which handles the linear constraints imposed by the divergence-free condition:
\begin{equation}
	\begin{bmatrix} \widehat{K} & \widehat{D_0} \\ \widehat{D_0}^\intercal & \mathbf{0} \end{bmatrix} \begin{bmatrix} F \\ \lambda \end{bmatrix} = \begin{bmatrix} 0 \\ b \end{bmatrix}.
\end{equation}
This formulation admits a closed-form solution for $F$ (and thus for $\mathbf{F}_\text{un}$) via back-substitution using the Schur complement:
\begin{equation}\begin{split}
	F &=  \widehat{K}^{-1} \widehat{D_0} \left( \widehat{D_0}^\intercal \widehat{K}^{-1} \widehat{D_0} \right)^{-1} b, \\
	\lambda &= - (\widehat{D_0}^\intercal \widehat{K}^{-1} \widehat{D_0})^{-1} b. \\
\end{split}\end{equation}
After reshaping $F$\footnote{That is, we reshape the final result to undo the vec($\cdot$) operation.} and substituting back into the GP recovery problem in \eqref{prob:gp-representer}, we arrive at the following optimization problem:
\begin{equation} \begin{split}
	\min_{\theta, \mathbf{u}_\text{un}} 
     & \sum_{e \in \mathcal{E}_\text{un}}  \mathbf{b}_e^T\left(\widehat{D_0}^\intercal \widehat{K}^{-1} \widehat{D_0}\right)^{-1} \mathbf{b}_e\\
	+ & \sum_{e \in \mathcal{E}_\text{obs}} \mathbf{F}_e^\intercal \left(\widehat{D_0}^\intercal \widehat{K}^{-1} \widehat{D_0}\right)^{-1} \mathbf{F}_e \\
	+ & \sum_{e \in \mathcal{E}} \log \det \left( K_e( \mathbf{u}_e,  \mathbf{u}_e) + \sigma_{\epsilon}^2 I \right).\label{eq:optimization_final}
\end{split} \end{equation}
This is now a tractable optimization problem: by minimizing over the model parameters $\theta$ and values on the unobserved vertices $\mathbf{u}_\text{un}$, we obtain predictions across the entire graph that obey both the boundary data and the global conservation constraint.

\subsection{Training details}\label{sec:training}

We minimize the MLE-based objective in \eqref{eq:optimization_final} with Adam \cite{KingmaBa2014}, a first-order optimizer capable of handling nonconvex, nonlinear optimization problems. We use an initial learning rate of $10^{-3}$ and a StepLR scheduler that scales the learning rate by 0.98 every 10,000 epochs.

The trainable model parameters $\theta$ consist of edge-specific length scale parameters $\ell_e$ for all $e\in E,$ and a global noise variance $\sigma_{\epsilon}^2.$ To enforce positivity, we optimize the logarithms $\tilde{\ell}_e=\log{\ell_e}$ and $\tilde{\sigma}^2_{\epsilon}=\log{\sigma^2_{\epsilon}}$, and exponentiate them inside the computational graph. We initialize each $\ell_e$ to the median pairwise distance between training inputs on that edge and $\sigma_{\epsilon}^2$ to $\exp(-20)$, or approximately $2e-9$.

\subsection{Inference}\label{sec:inference}

Solving the GP recovery problem provides a means of extrapolating to new boundary data that were not seen during training. Given test observations $\textbf{u}'_\text{obs}$ and $\mathbf{F}'_\text{obs}$ that were not seen during training, our goal is to compute $\mathbf{F}'_\text{un}$ by evaluating the trained GPs $f_e$ on $\mathbf{u}'_e$ for all $e\in\mathcal{E}$. However, in order to evaluate $f_e(\mathbf{u}'_\text{un}),$ we must first initialize the unknown vertex values, $\mathbf{u}'_\text{un}$. We do this randomly or, when possible, with an informed guess. In the absence of more detailed prior information, we guess the unknown vertex values in a consistent way by solving the linear system
\begin{align}\label{eq:extrapolate}
\delta_0^\intercal f_e(\mathbf{u}')=0,
\end{align}
for $\mathbf{u}'_\text{un}$ using a root-finding algorithm, such as Newton's method. This approach ensures that we recover a solution that obeys our global conservation law, such that the net flux on any interior vertex is zero. From there, we can evaluate
\begin{align}
    \mathbf{F}'=f_e(\mathbf{u}'),
\end{align}
for $\mathbf{F}'_\text{un}$. Thus, we are able to recover a mean posterior estimate for the values on all edges for any given boundary data.

What we ultimately obtain is a global Dirichlet-to-Neumann map. Note that we do not claim to recover the true data on the interior of the graph. Rather, our approach yields the RKHS-minimal model, which is in essence the smoothest physical model that is still consistent with the given boundary data. By solving the forward model, we infer internal voltages that obey the boundary data, allowing for interpolation of voltages and currents.

We impose physical constraints when performing Newton's method to ensure that interior vertex potentials lie between zero and the maximum boundary potential. Since solutions are, in general, non-unique depending on the numerical tolerance and level of noise assumed in the data, this type of constraint proves helpful in obtaining physically meaningful predictions.

\subsection{Error estimates}\label{sec:error}

A core strength of our approach is the rigorous uncertainty quantification made possible by working with GPs. In \Cref{lem:error}, we establish formal bounds on the worst-case error between the minimizer of our optimization problem $\hat{f}$ and the true function $f$ governing the dynamics of a given edge. Following standard practice in (nonlinear) optimal recovery, we assume that $f$ lies in the RKHS space associated with our chosen kernel $K(\cdot,\cdot)$ \cite{azarnavid_reproducing_2019,fiedler_practical_2023,kakade_worst-case_2005,leung2020introduction,reed_error_2024,zhou_gaussian_2023}. In other words, we assume that $f$ obeys the smoothness and structural properties encoded by $K(\cdot,\cdot),$ and can be approximated arbitrarily well in the RKHS norm by finite linear combinations of kernel sections $K(x_i,\cdot).$

\begin{restatable}[]{thm_restate}{error}\label{lem:error}
Let $\Omega\subseteq\mathbb{R}^d$, let $K:\Omega\times\Omega\rightarrow\mathbb{R}$ be a positive definite kernel with associated RKHS $\mathcal{H}_{K}$, and let $f$ be a true function from which we obtain data $(\mathbf{X},\mathbf{Y})=(x_i,y_i)_{i=1,...,N}$, where $y=f(x)+\epsilon$ for $\epsilon\sim\mathcal{N}(0,\sigma_{\epsilon}^2)$. If $\hat{f}\in\mathcal{H}_K$ is the minimizer of \eqref{prob:gp-recovery}, or the regressor to $f$ on $(\mathbf{X},\mathbf{Y})$, then we have:
\begin{align}
    \text{MSE}(x)\leq \sigma(x)^2 \norm{f}^2_{\mathcal{H}_K}+\sigma_{\epsilon}^2\norm{K(x,\mathbf{X})[K(\mathbf{X},\mathbf{X})+\sigma_{\epsilon}^2 I]^{-1}}_2^2,
\end{align}
for all $x\in\Omega,$ where $\sigma(x)^2=K(x,x)-K(x,\mathbf{X})[K(\mathbf{X},\mathbf{X})+\sigma_{\epsilon}^2 I]^{-1}K(\mathbf{X},x)$ is the conditional variance of the GP posterior at point $x$. We can also bound the pointwise absolute error by:
\begin{align}
    \lvert f(x)-\hat{f}(x)\rvert\leq \sigma(x)\norm{f}_{\mathcal{H}_K} +\sigma_{\epsilon}\norm{K(x,\mathbf{X})[K(\mathbf{X},\mathbf{X})+\sigma_{\epsilon}^2 I]^{-1}}_2\sqrt{2\log(2/\delta)},
\end{align}
with probability at least $1-\delta$.
\end{restatable}
\begin{proof}
We assume that the true function $f$ lies in the Reproducing Kernel Hilbert Space (RKHS) $\mathcal{H}_K$ associated with a kernel $K(\cdot,\cdot)$. The reproducing property of the RKHS $\mathcal{H}_{\mathcal{K}}$ states that for any function $f\in\mathcal{H_K}$ and any point $x$, the evaluation of $f$ at $x$ may be expressed as an inner product with the kernel function $K(x,\cdot)$:
$$f(x)=	\langle f,K(x,\cdot)\rangle_{\mathcal{H}_{K}}.$$ We approximate $f$ by the regressor $\hat{f}\in\mathcal{H}_K,$ the kernel method solution to the optimal recovery problem. That is, $\hat{f}$ is the RKHS-minimal function satisfying 
\begin{align}
    \hat{f}(\cdot)&=K(\cdot,\mathbf{X})[K(\mathbf{X},\mathbf{X})+\sigma_{\epsilon}^2 I]^{-1}\mathbf{Y}.
\end{align}
By construction, we can equivalently express $\hat{f}$ as a linear combination of the kernel functions centered at the training points:
\begin{align}
    \hat{f}(\cdot)=\sum_{i=1}^N a_i K(x_i,\cdot),
\end{align}
for some finite coefficients $a_i,...,a_N.$ Solving for the optimal $a_i$ using arguments similar to those in the proof of \Cref{thm:inner_min}, we find $a=[K(\mathbf{X},\mathbf{X})+\sigma_{\epsilon}^2I]^{-1}\mathbf{Y}$. Rearranging, we obtain the form:
\begin{align}\label{eq:interpolant_innerproduct}
    \hat{f}(x)&=\phi(x)\cdot y,\nonumber\\
    &=\phi(x)\cdot[f(x)+\epsilon],\nonumber\\
    &=\sum_{i=1}^N \phi_i(x)f(x_i)+\sum_{i=1}^N \phi_i(x)\epsilon_i,
\end{align}
where $\phi(x)=K(x,\mathbf{X})[K(\mathbf{X},\mathbf{X})+\sigma_{\epsilon}^2I]^{-1}$ is a $1\times N$ vector of weights, $K(x,\mathbf{X})$ is a $1\times N$ vector with entries $K(x,x_i)$ for $i\in[N]$, and $K(\mathbf{X},\mathbf{X})$ is an $N\times N$ matrix with entries $K(x_i,x_j)$ for $i,j\in[N]$. 

Consider the pointwise error: $$e(x)=f(x)-\hat{f}(x),$$ which decomposes into a deterministic component and a stochastic noise component:
\begin{align}
    e(x)&=f(x)-\Bigg(\sum_{i=1}^N \phi_i(x)f(x_i)+\sum_{i=1}^N \phi_i(x)\epsilon_i\Bigg),\\
    &=\Bigg(f(x)-\sum_{i=1}^N \phi_i(x)f(x_i)\Bigg) - \sum_{i=1}^N \phi_i(x)\epsilon_i,\\
    &=e_{\text{det}}(x)-e_{\text{noise}}(x).
\end{align}
Since the $\epsilon_i$ are unknown, we evaluate the error in expectation by considering the mean squared error (MSE):
\begin{align}
    \text{MSE}(x)&=\mathbf{E}[e(x)^2]=\mathbf{E}\big[\big(e_\text{det}(x)-e_\text{noise}(x)\big)^2\big],\\
    &=e_\text{det}(x)^2+\text{Var}[e_\text{noise}(x)].\label{eq:first_MSE}
\end{align}
Since $e_\text{noise(x)}$ has zero mean, the cross term in \eqref{eq:first_MSE} vanishes.

To evaluate $e_{\text{det}}$, we rearrange and apply the reproducing property of the RKHS, substituting $f(x_i)=\langle f,K(x_i,\cdot)\rangle_{\mathcal{H}_{K}}$, to obtain a single inner product,
\begin{align}
    e_{\text{det}}&=\langle f,K(x,\cdot)\rangle_{\mathcal{H}_{K}}-\sum_{i=1}^N\phi_i(x)\langle f,K(x_i,\cdot)\rangle_{\mathcal{H}_{K}},\\
    &=\langle f,K(x,\cdot)\rangle_{\mathcal{H}_{K}}-\langle f, \sum_{i=1}^N\phi_i(x)K(x_i,\cdot)\rangle_{\mathcal{H}_{K}},\\
    &=\langle f,K(x,\cdot)\rangle_{\mathcal{H}_{K}}-\langle f, K(x,\mathbf{X})[K(\mathbf{X},\mathbf{X})+\sigma_{\epsilon}^2I]^{-1} K(\mathbf{X}, \cdot)\rangle_{\mathcal{H}_{K}},\\
    &=\langle f, K(x,\cdot)- K(x,\mathbf{X})[K(\mathbf{X},\mathbf{X})+\sigma_{\epsilon}^2I]^{-1} K(\mathbf{X}, \cdot)\rangle_{\mathcal{H}_{K}}.
\end{align}
Applying the Cauchy-Schwarz inequality in $\mathcal{H}_{K}$ yields
\begin{align}
     \lvert e_{\text{det}}\rvert &\leq \norm{f}_{\mathcal{H}_{K}}\cdot\norm{K(x,\cdot)- K(x,\mathbf{X})[K(\mathbf{X},\mathbf{X})+\sigma_{\epsilon}^2I]^{-1} K(\mathbf{X}, \cdot)}_{\mathcal{H}_{K}}.
\end{align}
The final step to bound $e_\text{det}$ is to compute the RKHS norm of the residual which we denote by $r_x(\cdot)=K(x,\cdot)- K(x,\mathbf{X})[K(\mathbf{X},\mathbf{X})+\sigma_{\epsilon}^2I]^{-1} K(\mathbf{X}, \cdot)$. We compute its squared norm by expanding the inner product:
\begin{align}
\norm{r_x}^2_{\mathcal{H}_{K}}& =\langle K(x,\cdot),K(x,\cdot)\rangle_{\mathcal{H}_{K}}-2\langle K(x,\cdot), K(x,\mathbf{X})[K(\mathbf{X},\mathbf{X})+\sigma_{\epsilon}^2I]^{-1} K(\mathbf{X}, \cdot)\rangle_{\mathcal{H}_{K}} \\
&\quad + \langle K(x,\mathbf{X})[K(\mathbf{X},\mathbf{X})+\sigma_{\epsilon}^2I]^{-1} K(\mathbf{X}, \cdot),K(x,\mathbf{X})[K(\mathbf{X},\mathbf{X})+\sigma_{\epsilon}^2I]^{-1} K(\mathbf{X}, \cdot)\rangle_{\mathcal{H}_{K}}.
\end{align}
By the reproducing property, we make the following simplifications:
\begin{align}
    \langle K(x,\cdot),K(x,\cdot) \rangle_{\mathcal{H}_{K}} &= K(x,x),\\
    \langle K(x,\cdot), K(x,\mathbf{X})[K(\mathbf{X},\mathbf{X})+\sigma_{\epsilon}^2I]^{-1} K(\mathbf{X}, \cdot)\rangle_{\mathcal{H}_{K}} &= K(x,\mathbf{X})[K(\mathbf{X},\mathbf{X})+\sigma_{\epsilon}^2I]^{-1} K(\mathbf{X}, x),\\
    \langle K(x,\mathbf{X})[K(\mathbf{X},\mathbf{X})+\sigma_{\epsilon}^2I]^{-1} K(\mathbf{X}, \cdot),K(x,\mathbf{X})[K(\mathbf{X},\mathbf{X})+\sigma_{\epsilon}^2I]^{-1} K(\mathbf{X}, \cdot)\rangle_{\mathcal{H}_{K}} &\leq K(x,\mathbf{X})[K(\mathbf{X},\mathbf{X})+\sigma_{\epsilon}^2I]^{-1} K(\mathbf{X}, x).
\end{align}
We provide a proof of the third relation in \Cref{lem:reproducing_property_example}.\footnote{For the noiseless case, the relation is an equality.}
Combining the above, we obtain:
\begin{align}
    \norm{r_x}^2_{\mathcal{H}_{K}} \leq K(x,x) - K(x,\mathbf{X})[K(\mathbf{X},\mathbf{X})+\sigma_{\epsilon}^2I]^{-1} K(\mathbf{X}, x).
\end{align}
The right hand side above is exactly the conditional variance of the GP. Thus, we  bound $e_\text{det}$ by
\begin{align}
    \lvert e_\text{det}\rvert \leq \sigma(x)\norm{f}_{\mathcal{H}_{K}},
\end{align}
where $\sigma(x)$ is the conditional standard deviation of the GP.

We next consider $\text{Var}[e_\text{noise}(x)]:$
\begin{align}
    \text{Var}[e_\text{noise}(x)]&=\sum_{i=1}^N\phi_i(x)^2\text{Var}[\epsilon_i],\\
    &=\sigma_{\epsilon}^2\norm{\phi}_2^2,\\
    &=\sigma_{\epsilon}^2\norm{K(x,\mathbf{X})[K(\mathbf{X},\mathbf{X})+\sigma_{\epsilon}^2 I]^{-1}}_2^2.
\end{align}
Combining terms, we obtain the desired relation:
\begin{align}
    \text{MSE}(x)&\leq \sigma(x)^2 \norm{f}^2_{\mathcal{H}_K}+\sigma_{\epsilon}^2\norm{K(x,\mathbf{X})[K(\mathbf{X},\mathbf{X})+\sigma_{\epsilon}^2 I]^{-1}}_2^2.
\end{align}

Alternatively, we bound the absolute pointwise error:
\begin{align}
    \lvert f(x)-\hat{f}(x)\rvert &= \lvert e(x)\rvert,\\
    &\leq \lvert e_\text{det}(x)\rvert + \lvert e_\text{noise}(x)\rvert,\label{eq:pointwise_triangle}
\end{align}
where \eqref{eq:pointwise_triangle} follows from the triangle inequality. To bound $\lvert e_\text{noise}(x)\rvert$, we use a Gaussian tail bound to address the unknown noise $\epsilon_i$. Recall that $e_\text{noise}(x)$ is a weighted sum of independent Gaussian random variables. Standard results for Gaussian concentration inequalities \cite{vershynin2018} tell us that:
\begin{align}
    \mathbb{P}\bigg[\lvert Z\rvert &\geq t\bigg]\leq 2\exp\Bigg(\frac{-t^2}{2\sigma^2}\Bigg)\quad\text{ for }Z\sim\mathcal{N}(0,\sigma^2).
\end{align}
Setting the right hand side equal to $\delta$ and solving for $t$ yields a probabilistic guarantee on our stochastic error term:
\begin{align}
\mathbb{P}\bigg[|e_\text{noise}|&\leq\sigma_{\epsilon}\norm{\phi(x)}_2\sqrt{2\log(2/\delta)}\bigg]\geq 1-\delta.
\end{align}
Hence, with probability $1-\delta$, the absolute pointwise error is bounded by
\begin{align}
    \lvert f(x)-\hat{f}(x)\rvert\leq \sigma(x)\norm{f}_{\mathcal{H}_K} +\sigma_{\epsilon}\norm{\phi(x)}_2\sqrt{2\log(2/\delta)}.
\end{align}

\end{proof}

\begin{remark}
The analysis here does not account for uncertainty stemming from the inferred input to the GPs, and thus assumes that these inferred variables are accurate. Recall that the interior vertex values (and thus the inputs to the GP) are unknown and are estimated via the process described in \Cref{sec:inference}. As a potential extension, this additional source of uncertainty could be incorporated by linearizing the problem around the maximum a posteriori (MAP) estimate, as in \cite{chen_gaussian_2024}.
\end{remark}

The error bounds in \Cref{lem:error} make the common assumption that $f\in\mathcal{H}_K$. 
In practice, the true function $f$ might lie outside of the RKHS associated with our chosen kernel $K(\cdot,\cdot)$. In this case, often referred to as model misspecification, the induced RKHS should be interpreted as a ``hypothesis space,'' and our approach returns the function $\hat{f}\in\mathcal{H}_K$ closest to the projection of $f$ onto $\mathcal{H}_K$. For most commonly used kernels, $\mathcal{H}_K$ is dense in the space of continuous functions, and is thus rich enough to provide a suitable approximation of $f$ \cite{steinwart}. Furthermore, \cite{fiedler_practical_2023} show that many types of model misspecification either do not affect error estimate bounds at all, or only affect them to a minor extent.

\section{Experiments}\label{sec:results}
We evaluate the proposed method on three datasets of increasing complexity: a toy electrical circuit, a discrete subsurface-fracture network, and a synthetic arterial blood flow graph. Besides covering a diverse set of application domains, each dataset poses a unique challenge for the approach. In each example, we train on 10 samples.

\subsection{Toy circuit}
We first present results for a simple linear circuit consisting of three edges in series (left panel of \Cref{fig:toy}). We sample boundary voltages i.i.d. from a centered Gaussian distribution and solve the resulting Poisson problems to obtain the corresponding edge currents. We train on 10 such solutions, and reserve an additional collection of boundary conditions for testing. 

\begin{figure}[h]
\centering
\includegraphics[width=0.85\linewidth]{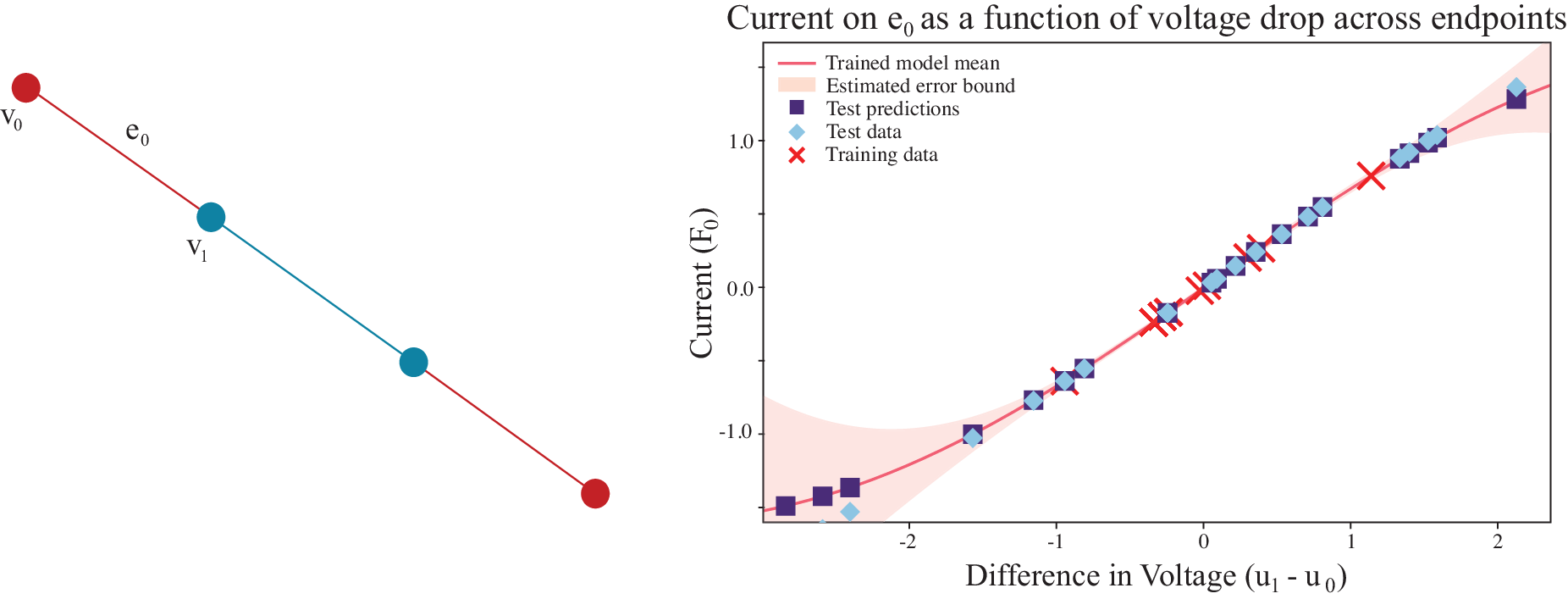}
\caption{Toy circuit problem. Left: a toy circuit with three edges in series. Right: inference results for leftmost boundary edge. Training points are marked in red, with the test predictions and target values in purple and blue, respectively. The pink shading indicates the estimated error bound from \Cref{lem:error} for $\delta=0.05,$ corresponding to $95\%$ confidence. Note that the error bound grows away from the cluster of training points.}\label{fig:toy}
\end{figure}

The primary output of interest from our approach is the learned Dirichlet-to-Neumann map: a set of edge-wise functions that map a pair of endpoint voltages $\mathbf{u}_e$ to a compatible current $\mathbf{F}_e$. Given a new set of boundary vertex values that were not seen during training, the model infers compatible values on interior vertices that yield physically feasible currents on the edges, following the inference strategy described in \Cref{sec:inference}. The right panel of \Cref{fig:toy} illustrates sample output for this process. Importantly, the bounds in \Cref{lem:error} provide an estimate of uncertainty along with the model predictions. In these and subsequent results, we demonstrate that the true test data consistently lies within the estimated 95\% confidence interval; in other words, our model is not only highly accurate, but also provides a trustworthy indication of where it may be more or less accurate.

\begin{figure}[h]
\centering
\includegraphics[width=0.6\linewidth]{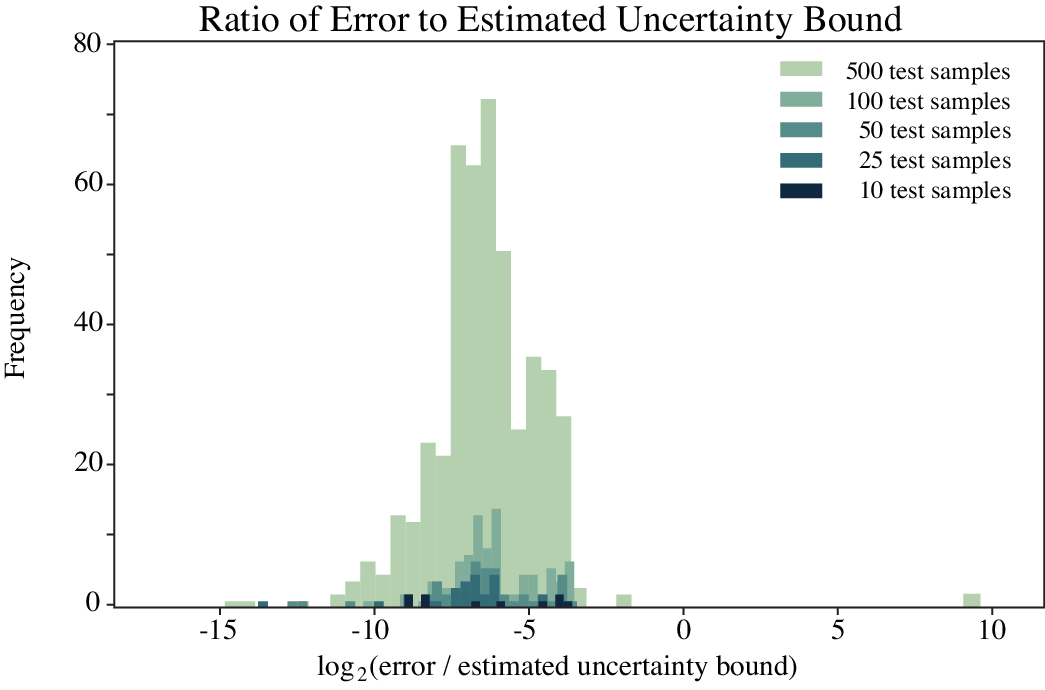}
\caption{Histogram of $\log_2(\lvert\text{error}\rvert/\lvert\text{bound}\rvert)$ for different test-set sizes (color-coded). Values below zero indicate that the estimated bound safely over-predicts the true error; only a few cases in the sample of size 500 lie above zero, as expected from the $95\%$ confidence level.} 
\label{fig:toy_scaling}
\end{figure}

\Cref{fig:toy_scaling} presents the results of a simple scaling test that demonstrates the behavior of our approach as the test set increases in size. Specifically, we show a histogram of the ratio of actual test error to the estimated error bound. The ratios concentrate between $[-3,-7],$ meaning that the true error is typically between 8 and 128 times smaller than the predicted worst-case bound. Only when 500 test points are sampled do a handful exceed the bound, an outcome consistent with the nominal failure probability $\delta=0.05$. In other words, the estimated error bound is slightly conservative for most realistic test sizes but remains probabilistically valid. 

\subsection{Subsurface fracture network problem}
Discrete fracture networks serve as preferential pathways for fluid flow and heat transport in subsurface systems. Understanding the dynamics of such systems is essential for applications such as carbon sequestration, geothermal energy, and hydraulic fracturing. However, observations of such systems are typically both sparse and noisy, rendering traditional approaches insufficient for proper study. Surrogate modeling with deep neural networks has been proposed as one potential approach to studying fracture network properties \cite{song_surrogate_2024}. In the current work, we explore Dirichlet-to-Neumann maps on graphs as an alternative approach.

We use the synthetic 2-D fracture network dataset of \cite{song_surrogate_2024}, generated by a physics-based Darcy-flow solver. 
The fracture geometry is generated according to the procedure detailed in \cite{delphine_fractures,song_surrogate_2024}, and is characterized by the percolation parameter, $p=9$, and aperture, $a=1.5$.  Fracture segments are designated as edges, and fracture intersections and terminations are designated as vertices. The resulting graph, depicted in \Cref{fig:DFN}, has $V=107$ vertices and $E=130$ edges, of which 17 are designated boundary boundary entities (nine injection sites with high hydraulic head and eight extraction sites with low hydraulic head).

We solve 256 realizations of steady-state, incompressible, single-phase Darcy flow, subject to mass conservation at every interior vertex and a hydraulic head drop $h_L\in[1,256]$ from the left-hand to the right-hand boundary. Because the volumetric flow rate follows the Poiseuille law and therefore depends on the gradient of hydraulic head rather than its magnitude, we opt to use the gradient of the endpoints, $\delta_0\mathbf{u}_e$ as the input to the GP for this problem rather than the raw endpoint values $\mathbf{u}_e$.

\begin{figure}
\centering
\includegraphics[width=0.95\linewidth]{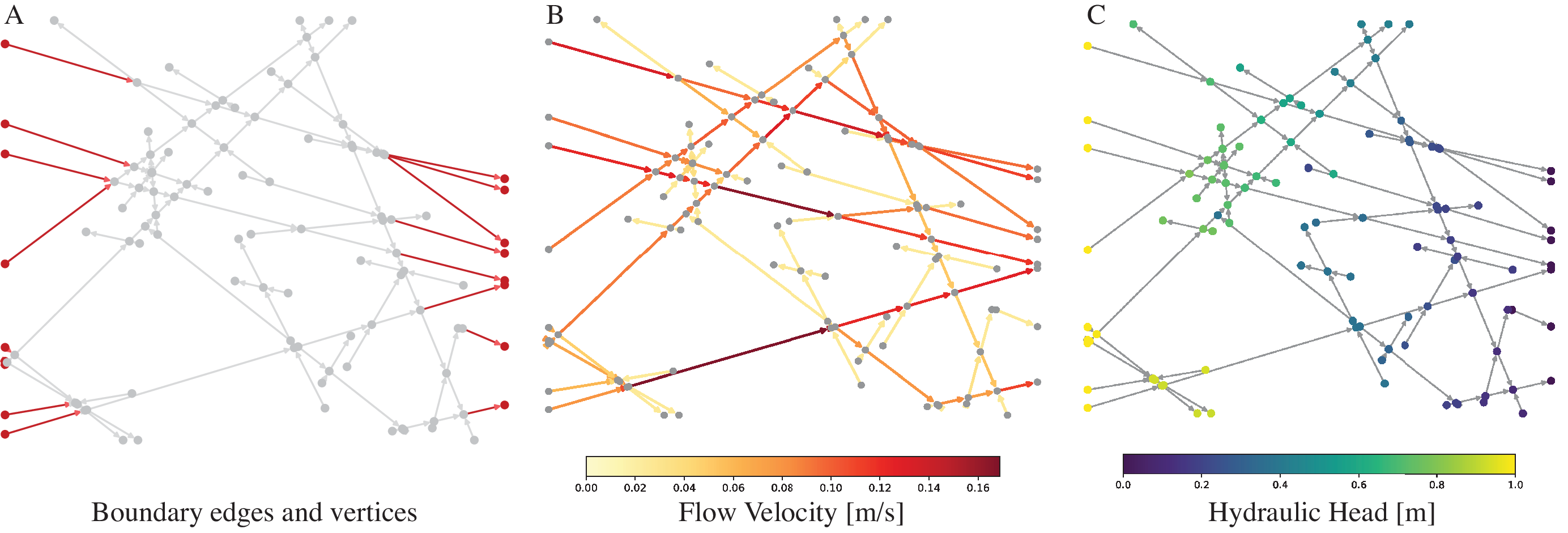}
\caption{Example subsurface fracture dataset. A) Graph topology with boundary edges and vertices in red; B) Flow velocity on edges; C) Hydraulic head on vertices. Flow proceeds predominantly along darker-red edges from yellow vertices (with high hydraulic head) to black vertices (with low hydraulic head).}\label{fig:DFN}
\end{figure}

\paragraph{Our approach recovers the Dirichlet-to-Neumann map on the boundary with high fidelity.}
\Cref{fig:DFN_results} presents a representative sample of results obtained for this dataset. Panel C shows that we achieve low levels of error in the flow velocity predictions on all boundary edges.\footnote{Note that the edges in the lower-left corner of \Cref{fig:DFN_results} (C) displaying higher error are not actually boundary edges, despite being located close to the boundary.} Panel F plots the learned functional map for a representative boundary edge. All 20 test predictions (blue squares) fall inside the $95\%$ credibility band, confirming well-calibrated uncertainty.

\paragraph{We obtain a physically meaningful prediction on the entire graph.} Despite only training on boundary data, our model recovers a physically meaningful and globally conservative prediction of hydraulic head and flow velocity across the entire domain. Panels A/B and D/E of \Cref{fig:DFN_results} compare predicted versus true hydraulic head (vertices) and flow velocity (edges). Discrepancies remain small and confined to interior cycles where multiple admissible flow configurations can satisfy the same boundary conditions (Panel C). This behavior is expected: global mass conservation enforces a unique potential field but may still allow localized flow permutations with negligible impact on boundary fluxes. 

\begin{figure}[h!]
\centering
\includegraphics[width=\linewidth]{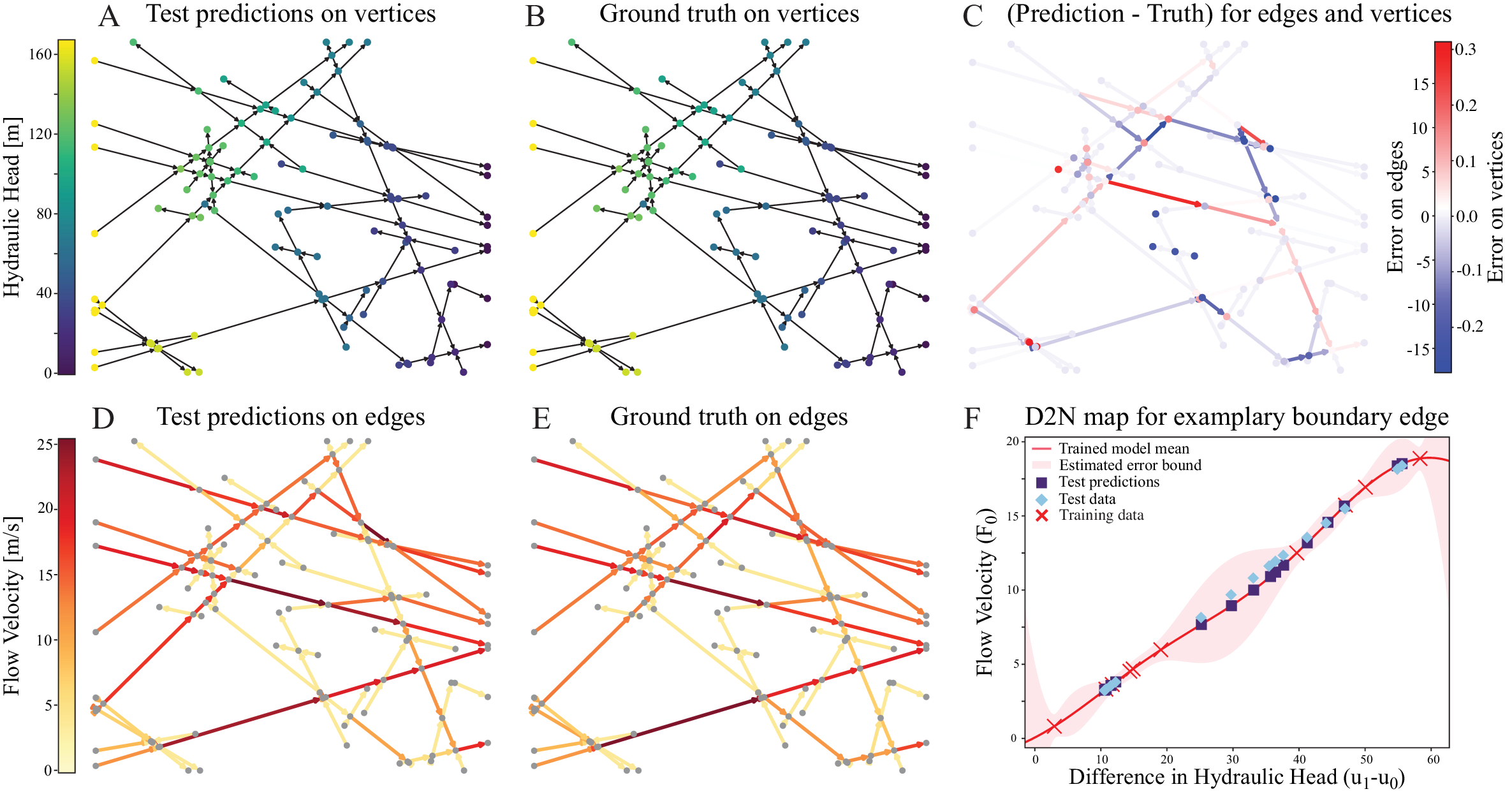}
\caption{Example results for the subsurface fracture network problem. A) Predicted hydraulic head; B) Ground truth hydraulic head; C) Error (true-predicted) on edges and vertices; D) Predicted flow velocity; E) Ground truth flow velocity; F) Inferred D2N relationship for boundary edge 25, with posterior mean in red, $95\%$ confidence band in pink, training data in red, test predictions in purple, and ground truth test values in blue.}
\label{fig:DFN_results}
\end{figure}

\newpage
\subsection{Arterial blood flow problem}
The high computational cost of building high fidelity models with full field patient-specific data makes arterial blood flow another appealing test case for this approach. Arterial blood flow is a challenging test case because it exhibits both nonlinear dynamics and potentially complicated geometry.
Our dataset consists of arterial graphs generated as training data for the GNN reduced order model developed in \cite{pegolotti_learning_2023}. Data generation involves first running the 3D hemodynamic simulation SimVascular \cite{updegrove2017simvascular}, then converting the data to a 1D representation and ultimately into a graph. The full data generation process is described in Section 4 of \cite{pegolotti_learning_2023}.

\begin{figure}[h!]
\centering
\includegraphics[width=0.7\linewidth]{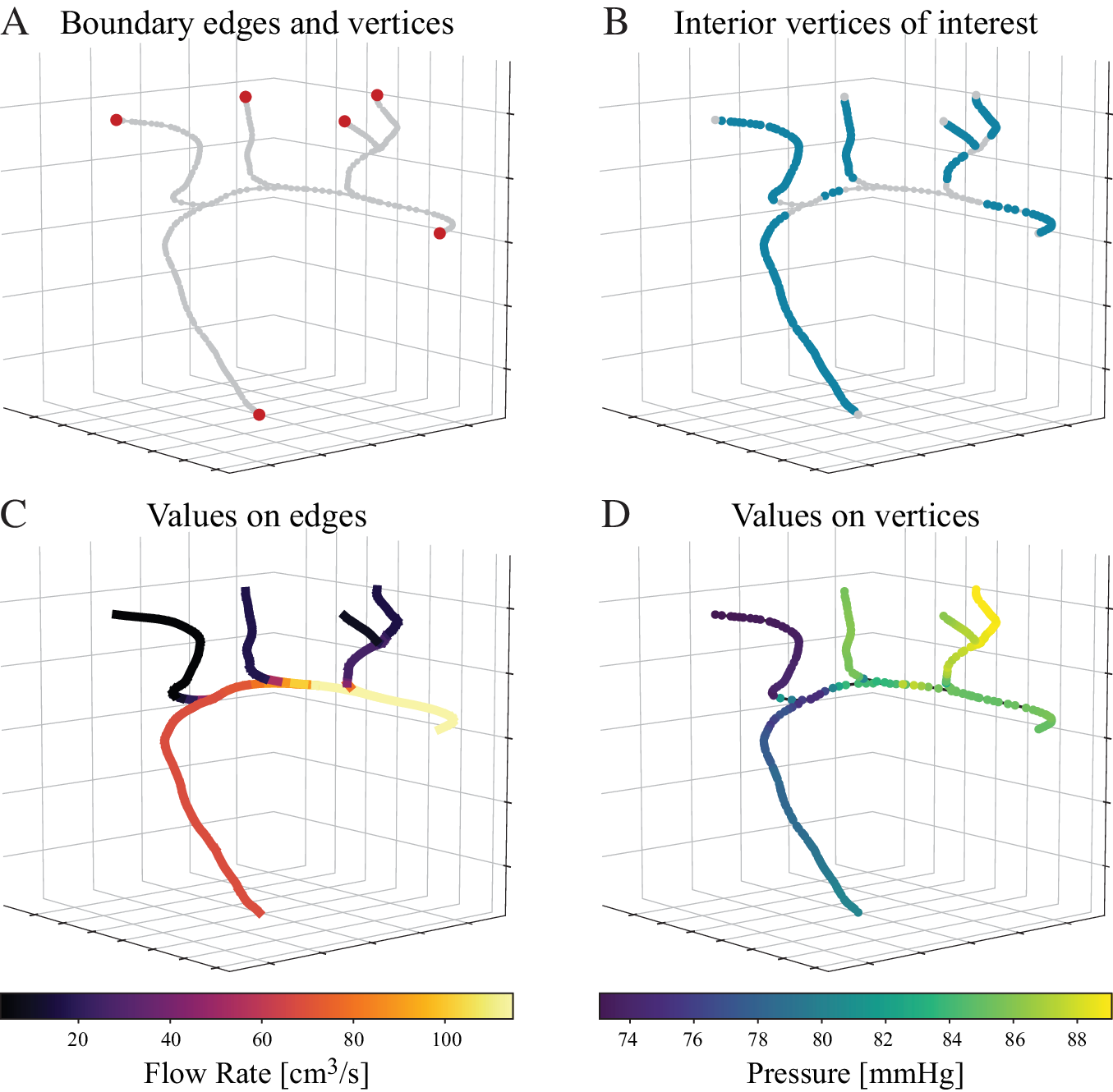}
\caption{Example arterial dataset. A) Topology of the dataset, highlighting boundary vertices and edges in red; B) Topology of the dataset, highlighting interior vertices and edges of interest in blue; C) Flow rate on graph edges; D) Pressure on graph vertices. Blood flows from the right-most boundary through the rest of the arterial branches.}\label{fig:artery}
\end{figure}

In the original dataset, vertices represent data collection sites for both pressure and flow rate. There are four types of vertices: inlet, outlet, junction, and branch vertices. We consider inlet and outlet vertices to be observed ``boundary'' vertices, and all other vertices to be unobserved ``interior'' vertices. Since we are interested in representing flow rate as values on edges, we average the flow rate of each edge's endpoint vertices to approximate the flow rate on that edge at each timestep.

Because the approach described in \cite{pegolotti_learning_2023} calculated the artery aperture using a cutting-plane method, the dataset can be nonphysical (and hence, non-conservative) in areas near intersections, since a single plane may cut more than one arterial branch. For this reason, when we consider results on the interior of the graph, we focus on the ``branch'' vertices where this data fidelity issue does not arise (highlighted in \Cref{fig:artery} Panel B).

\Cref{fig:artery} depicts an example arterial graph topology and an example dataset on this graph. The graph consists of 271 vertices and 270 edges, including six boundary vertices and edges. The simulation data consists of 424 timesteps, which we consider as individual data points.  Note that unlike the subsurface fracture network, this graph is a tree: for any pair of vertices in the graph, there is only one path connecting them. Thus, the arterial blood flow problem presents a very different challenge for our approach.

\paragraph{2D input to the GP is necessary for applications where compressibility and momentum play a role.} Unlike with subsurface Darcy flow, the behavior of arterial blood flow requires extra care for model setup. In particular, the flow rate on an edge depends not only on the pressure gradient across its endpoints, but also on the magnitude of the pressures and the phase of the cardiac cycle. For this reason, we opted to use a 2D input to the Gaussian process. Whereas in the previous two examples our GPs were a function of $\delta_0 \mathbf{u}_e$ (the gradient across the endpoint vertices of an edge), now they are a function of $\mathbf{u}_e$ (the raw values on both endpoints).

\begin{figure}[h!]
\centering
\includegraphics[width=\linewidth]{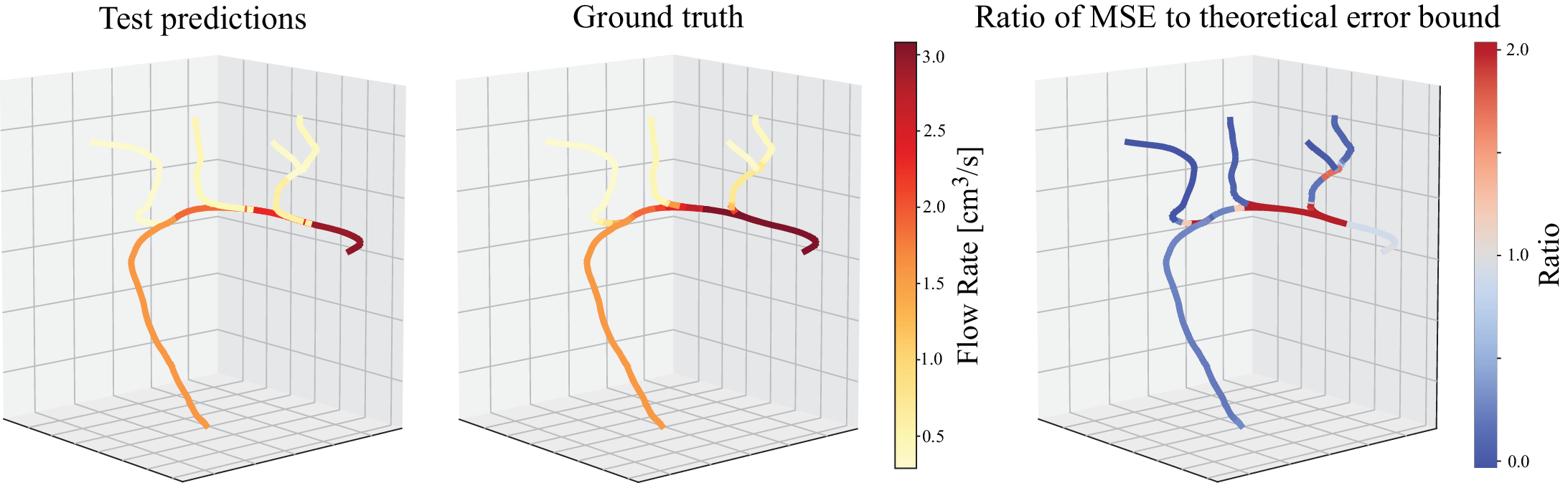}
\caption{Example results for the arterial blood flow problem. Left: test predictions on edges for an example test set of boundary conditions; Center: ground truth for the same set of boundary conditions; Right: ratio of realized error---calculated as mean squared error (MSE) across test samples---to the error bound from \Cref{lem:error}. We use a diverging color palette to highlight where the MSE is within the predicted levels of uncertainty. On the boundary and the interior vertices of interest (that is, where the ground truth is trustworthy---see \Cref{fig:artery}), the MSE is well within the predicted bound. }\label{fig:artery_results}
\end{figure}

\paragraph{Our approach returns a conservative model even when the data itself is not conservative.} The cutting-plane artifact in \cite{pegolotti_learning_2023} breaks local mass conservation near junctions; our model, by construction, enforces a globally conservative flow field that minimizes the RKHS norm subject to the boundary data. Consequently, the largest discrepancies between the predictions and truth appear exactly in those non-conservative regions, as visualized by the blue and red shading in the rightmost panel of \Cref{fig:artery_results}. For branch vertices---where the true data are conservative---errors are small and are captured by the estimated uncertainty bound. This is because our approach constructs a conservative model obeying the boundary data regardless of whether the ground truth data is actually locally conservative everywhere. It is also important to note that we expect our estimated uncertainty bound to underestimate the error on the interior due to the fact that the bound assumes the GP input values on the interior are accurate, when they are actually inferred estimates (see \Cref{sec:inference}).

\section{Discussion}

\subsection{Future work}

While our derivation ensures that we obtain an RKHS-minimal solution, it does not imply uniqueness. With graphs such as the discrete fracture network that involve cycles, the flow configuration satisfying a given set of boundary data is inherently non-unique. If a particular application requires a unique solution, there are a few possible approaches to ensuring this. For example, one might perform a Hodge decomposition of the flux, keeping only the curl-free component. This would, in theory, eliminate any indeterminacy caused by cycles without affecting boundary edges or other edges that are not involved in the cycles. It would also be interesting to explore the effect of adding observations at select vertices and edges on the interior of the graph. This could address the non-uniqueness issue in the case of cycles, and would also likely serve to further improve the accuracy of the coupled models.

We assumed small Gaussian noise in our observations. The effect of this assumption is the regularization of the kernel matrix with a small diagonal term $\sigma_{\epsilon}^2 I$, referred to as a ``nugget.'' For simplicity, we have assumed a uniform value for $\sigma_{\epsilon}^2,$ but in reality the spectral properties of the different edges of our graph (and thus the different blocks of the kernel matrix) may differ widely. An adaptive nugget term, such as that proposed by \cite{chen_solving_2021}, could more optimally balance accuracy and stability across the different edges of the graph.

Our current treatment does not account for the uncertainty stemming from the inferred input values to the GPs. We assume that the interior vertex values are unknown and estimated via the procedure described in \Cref{sec:inference}, but the error estimates provided in \Cref{sec:error} do not account for this. To incorporate this additional source of uncertainty, future work could explore linearizing the problem around the MAP, for example via a Laplace approximation as in \cite{chen_gaussian_2024}.

Finally, future work could explore learning the kernel $K(\cdot,\cdot)$ that minimizes the error bounds in \Cref{lem:error}. For example, this could involve an additional nested training loop minimizing the error bound over a family of candidate kernels. Recent work in this area has shown promising results using kernel flows, a variant of cross-validation \cite{learningkernels2,learningkernels}.

\subsection{Conclusion}
We have presented a novel approach to learning Dirichlet-to-Neumann maps on graphs that combines Gaussian processes and discrete exterior calculus. Our method produces robust uncertainty estimates along with its predictions, a desirable consequence of using Gaussian processes and the CGC framework as the backbone of our approach. By minimizing the RKHS norm with an MLE penalty, we ensure that our surrogate obeys our conservation law without excessive complexity. The tight match between prediction and ground truth in each of our experiments---together with rigorous error bounds---demonstrates that our combined DEC and CGC framework successfully learns a global Dirichlet-to-Neumann surrogate from limited boundary observations. In practice, the method can replace repeated sub-domain solves inside multiscale simulators while still providing trustworthy uncertainty estimates needed for verification and decision-making.

\section{Acknowledgements}
The authors appreciate the support from the Department of Energy under award number DE-SC0023163 (SEA-CROGS: Scalable, Efficient, and Accelerated Causal Reasoning Operators, Graphs, and Spikes for Earth and Embedded Systems). AMP acknowledges support from the Stanford Graduate Fellowship (SGF) and Stanford Enhancing Diversity in Graduate Education (EDGE) fellowship. HO acknowledges support from the Vannevar Busch Fellowship program (ONR award number N000142512035). The authors thank Delphine Roubinet and Luca Pegolotti for providing the data sets for the paper's numerical experiments.

\bibliographystyle{siamplain}
\bibliography{references}

@article{owhadi2022computational,
  title={Computational graph completion},
  author={Owhadi, Houman},
  journal={Research in the Mathematical Sciences},
  volume={9},
  number={2},
  pages={27},
  year={2022},
  publisher={Springer}
}

@article{farhat2001feti,
  title={{FETI-DP}: a dual--primal unified {FETI} method—part {I}: A faster alternative to the two-level {FETI} method},
  author={Farhat, Charbel and Lesoinne, Michel and LeTallec, Patrick and Pierson, Kendall and Rixen, Daniel},
  journal={International Journal for Numerical Methods in Engineering},
  volume={50},
  number={7},
  pages={1523--1544},
  year={2001},
  publisher={Wiley Online Library}
}

@article{steinwart,
author = {Steinwart, Ingo},
title = {On the influence of the kernel on the consistency of support vector machines},
year = {2002},
volume = {2},
doi = {10.1162/153244302760185252},
journal = {J. Mach. Learn. Res.},
pages = {67–93},
numpages = {27}
}

@article{learningkernels2,
 ISSN = {13645021, 14712946},
 URL = {https://www.jstor.org/stable/27097827},
 author = {B. Hamzi and R. Maulik and H. Owhadi},
 journal = {Proceedings: Mathematical, Physical and Engineering Sciences},
 number = {2252},
 pages = {pp. 1--18},
 publisher = {Royal Society},
 title = {Simple, low-cost and accurate data-driven geophysical forecasting with learned kernels},
 urldate = {2025-05-19},
 volume = {477},
 year = {2021}
}

@book{Owhadi_Scovel_2019, place={Cambridge}, series={Cambridge Monographs on Applied and Computational Mathematics}, title={Operator-Adapted Wavelets, Fast Solvers, and Numerical Homogenization: From a Game Theoretic Approach to Numerical Approximation and Algorithm Design}, publisher={Cambridge University Press}, author={Owhadi, Houman and Scovel, Clint}, year={2019}, collection={Cambridge Monographs on Applied and Computational Mathematics}}

@article{learningkernels,
title = {Learning dynamical systems from data: A simple cross-validation perspective, part I: Parametric kernel flows},
journal = {Physica D: Nonlinear Phenomena},
volume = {421},
pages = {132817},
year = {2021},
issn = {0167-2789},
doi = {https://doi.org/10.1016/j.physd.2020.132817},
url = {https://www.sciencedirect.com/science/article/pii/S0167278920308186},
author = {Boumediene Hamzi and Houman Owhadi},
}

@book{vershynin2018,
  author = {Vershynin, Roman},
  title = {{High-Dimensional Probability: An Introduction with Applications in Data Science}},
  year = {2018},
  publisher = {Cambridge University Press},
  address = {Cambridge, UK},
  series = {Cambridge Series in Statistical and Probabilistic Mathematics},
  volume = {47},
  isbn = {978-1-108-41519-4},
  url = {https://www.math.uci.edu/~rvershyn/papers/HDP-book/HDP-book.pdf}
}

@article{hall2021ginns,
  title={{GINNs}: Graph-informed neural networks for multiscale physics},
  author={Hall, Eric J and Taverniers, S{\o}ren and Katsoulakis, Markos A and Tartakovsky, Daniel M},
  journal={Journal of Computational Physics},
  volume={433},
  pages={110192},
  year={2021},
  publisher={Elsevier}
}

@article{sockwell2020interface,
  title={Interface Flux Recovery coupling method for the ocean--atmosphere system},
  author={Sockwell, K Chad and Peterson, Kara and Kuberry, Paul and Bochev, Pavel and Trask, Nat},
  journal={Results in Applied Mathematics},
  volume={8},
  pages={100110},
  year={2020},
  publisher={Elsevier}
}

@article{arbogast2007multiscale,
  title={A multiscale mortar mixed finite element method},
  author={Arbogast, Todd and Pencheva, Gergina and Wheeler, Mary F and Yotov, Ivan},
  journal={Multiscale Modeling \& Simulation},
  volume={6},
  number={1},
  pages={319--346},
  year={2007},
  publisher={SIAM}
}

@incollection{bernardi1993domain,
  title={Domain decomposition by the mortar element method},
  author={Bernardi, Christine and Maday, Yvon and Patera, Anthony T},
  booktitle={Asymptotic and numerical methods for partial differential equations with critical parameters},
  pages={269--286},
  year={1993},
  publisher={Springer}
}

@article{cockburn2009unified,
  title={Unified hybridization of discontinuous {Galerkin}, mixed, and continuous {Galerkin} methods for second order elliptic problems},
  author={Cockburn, Bernardo and Gopalakrishnan, Jayadeep and Lazarov, Raytcho},
  journal={SIAM Journal on Numerical Analysis},
  volume={47},
  number={2},
  pages={1319--1365},
  year={2009},
  publisher={SIAM}
}

@article{natarajan1995domain,
  title={Domain decomposition using spectral expansions of {Steklov}--{Poincar{\'e}} operators},
  author={Natarajan, Ramesh},
  journal={SIAM Journal on Scientific Computing},
  volume={16},
  number={2},
  pages={470--495},
  year={1995},
  publisher={SIAM}
}

@article{jiang2024structure,
  title={A structure-preserving domain decomposition method for data-driven modeling},
  author={Jiang, Shuai and Actor, Jonas and Roberts, Scott and Trask, Nathaniel},
  journal={arXiv preprint arXiv:2406.05571},
  year={2024}
}

@article{bochev2024dynamic,
  title={Dynamic flux surrogate-based partitioned methods for interface problems},
  author={Bochev, Pavel and Owen, Justin and Kuberry, Paul and Connors, Jeffrey},
  journal={Computer Methods in Applied Mechanics and Engineering},
  volume={429},
  pages={117115},
  year={2024},
  publisher={Elsevier}
}

@article{azarnavid_reproducing_2019,
	title = {A reproducing kernel {Hilbert} space approach in meshless collocation method},
	volume = {38},
	doi = {10.1007/s40314-019-0838-0},
	number = {2},
	journal = {Computational and Applied Mathematics},
	author = {Azarnavid, Babak and Emamjome, Mahdi and Nabati, Mohammad and Abbasbandy, Saeid},
	year = {2019},
	pages = {72},
}

@inproceedings{dvorak2005unifying,
  title={A unifying framework for systems modeling, control systems design, and system operation},
  author={Dvorak, Daniel L and Indictor, Mark B and Ingham, Michel D and Rasmussen, Robert D and Stringfellow, Margaret V},
  booktitle={2005 IEEE International Conference on Systems, Man and Cybernetics},
  volume={4},
  pages={3648--3653},
  year={2005},
  organization={IEEE}
}

@article{leung2020introduction,
  title={An introduction to the {E3SM} special collection: Goals, science drivers, development, and analysis},
  author={Leung, L Ruby and Bader, David C and Taylor, Mark A and McCoy, Renata B},
  journal={Journal of Advances in Modeling Earth Systems},
  volume={12},
  number={11},
  pages={e2019MS001821},
  year={2020},
  publisher={Wiley Online Library}



}

@misc{fiedler_practical_2023,
	title = {Practical and Rigorous Uncertainty Bounds for {Gaussian} Process Regression},
	doi = {10.48550/arXiv.2105.02796},
	author = {Fiedler, Christian and Scherer, Carsten W. and Trimpe, Sebastian},
	year = {2023},
}

@article{ibvp_2000,
  title = {Solving inverse initial-value, boundary-value problems via genetic algorithm},
  author = {Karr, Charles L. and Yakushin, Igor and Nicolosi, Keith},
  volume = {13},
  journal = {Eng. App. AI},
  year = {2000},
  pages = {625-633}
}

@article{zhou_gaussian_2023,
	title = {Gaussian processes with errors in variables: theory and computation},
	volume = {24},
	issn = {1532-4435},
	shorttitle = {Gaussian processes with errors in variables},
	number = {1},
	journal = {J. Mach. Learn. Res.},
	author = {Zhou, Shuang and Pati, Debdeep and Wang, Tianying and Yang, Yun and Carroll, Raymond J.},
	month = jan,
	year = {2023},
	pages = {87:3959--87:4011},
}

@inproceedings{kakade_worst-case_2005,
	title = {Worst-Case Bounds for {Gaussian} Process Models},
	volume = {18},
	booktitle = {Advances in Neural Information Processing Systems},
	publisher = {MIT Press},
	author = {Kakade, Sham M. and Seeger, Matthias W. and Foster, Dean P.},
	year = {2005}
}

@misc{reed_error_2024,
	title = {Error Bounds For {Gaussian} Process Regression Under Bounded Support Noise With Applications To Safety Certification},
	doi = {10.48550/arXiv.2408.09033},
	author = {Reed, Robert and Laurenti, Luca and Lahijanian, Morteza},
	year = {2024},
}

@INPROCEEDINGS{reconstruction,
  author={Kong, Linghe and Xia, Mingyuan and Liu, Xiao-Yang and Wu, Min-You and Liu, Xue},
  booktitle={2013 Proceedings IEEE INFOCOM}, 
  title={Data loss and reconstruction in sensor networks}, 
  year={2013},
  volume={},
  number={},
  pages={1654-1662},
  doi={10.1109/INFCOM.2013.6566962}}

@misc{chen_gaussian_2024,
	title = {Gaussian Measures Conditioned on Nonlinear Observations: {Consistency}, {MAP} Estimators, and Simulation},
	doi = {10.48550/arXiv.2405.13149},
	urldate = {2025-01-11},
	publisher = {arXiv},
	author = {Chen, Yifan and Hosseini, Bamdad and Owhadi, Houman and Stuart, Andrew M.},
	year = {2024},
}

@inproceedings{KingmaBa2014,
  title={Adam: A Method for Stochastic Optimization},
  author={Kingma, Diederik P. and Ba, Jimmy},
  booktitle={Proceedings of the 3rd International Conference on Learning Representations (ICLR)},
  year={2014},
  url={https://arxiv.org/abs/1412.6980}
}

@misc{chen_solving_2021,
	title = {Solving and Learning Nonlinear {PDEs} with {Gaussian} Processes},
	doi = {10.48550/arXiv.2103.12959},
	author = {Chen, Yifan and Hosseini, Bamdad and Owhadi, Houman and Stuart, Andrew M.},
	year = {2021},
}

@phdthesis{hirani_thesis_2003,
  author       = {Hirani, Anil N.},
  title        = {Discrete Exterior Calculus},
  school       = {California Institute of Technology},
  year         = {2003},
  address      = {Pasadena, CA},
  type         = {{Ph.D.} Thesis},
  url          = {https://thesis.library.caltech.edu/50/},
}

@incollection{desbrun_kanso_tong_2008,
  author       = {Desbrun, Mathieu and Kanso, Eitan and Tong, Yiying},
  title        = {Discrete Differential Forms for Geometry Processing},
  booktitle    = {Discrete Differential Geometry},
  editor       = {Bobenko, Alexander I. and Schr{\"o}der, Peter and Sullivan, John M. and Ziegler, G{\"u}nter M.},
  year         = {2008},
  chapter      = {13},
  pages        = {287--324},
  publisher    = {Springer},
  address      = {Berlin, Heidelberg},
}

@article{delphine_fractures,
    author = {Demirel, Serdar and Irving, James and Roubinet, Delphine},
    title = {Comparison of {REV} size and tensor characteristics for the electrical and hydraulic conductivities in fractured rock},
    journal = {Geophysical Journal International},
    volume = {216},
    number = {3},
    pages = {1953-1973},
    year = {2018},
    month = {12},
    doi = {10.1093/gji/ggy537}
}

@article{Micchelli1976,
author = {Micchelli, C. A. and Rivlin, T. J. and Winograd, S.},
journal = {Numerische Mathematik},
pages = {191-200},
title = {The Optimal Recovery of Smooth Function},
volume = {26},
year = {1976}
}

@book{rasmussen_gaussian_2008,
	address = {Cambridge, MA},
	edition = {3rd},
	title = {Gaussian Processes for Machine Learning},
	publisher = {MIT Press},
	author = {Rasmussen, Carl Edward and Williams, Christopher K. I.},
	year = {2008},
}

@article{donoho_statistical_1994,
	title = {Statistical Estimation and Optimal Recovery},
	volume = {22},
	doi = {10.1214/aos/1176325367},
	number = {1},
	journal = {The Annals of Statistics},
	author = {Donoho, David L.},
	year = {1994},
	pages = {238--270}
}

@incollection{micchelli_survey_1977,
	address = {Boston, MA},
	title = {A Survey of Optimal Recovery},
	booktitle = {Optimal Estimation in Approximation Theory},
	publisher = {Springer US},
	author = {Micchelli, C. A. and Rivlin, T. J.},
	year = {1977},
	doi = {10.1007/978-1-4684-2388-4_1},
	pages = {1--54},
}

@article{updegrove2017simvascular,
  title={{SimVascular}: an open source pipeline for cardiovascular simulation},
  author={Updegrove, Adam and Wilson, Nathan M and Merkow, Jameson and Lan, Hongzhi and Marsden, Alison L and Shadden, Shawn C},
  journal={Annals of Biomedical Engineering},
  volume={45},
  pages={525--541},
  year={2017},
  publisher={Springer}
}

@misc{pegolotti_learning_2023,
	title = {Learning Reduced-Order Models for Cardiovascular Simulations with Graph Neural Networks},
	author = {Pegolotti, Luca and Pfaller, Martin R. and Rubio, Natalia L. and Ding, Ke and Brufau, Rita Brugarolas and Darve, Eric and Marsden, Alison L.},
	year = {2023},
	note = {arXiv:2303.07310},
}

@article{song_surrogate_2024,
	title = {Surrogate models of heat transfer in fractured rock and their use in parameter estimation},
	volume = {183},
	doi = {https://doi.org/10.1016/j.cageo.2023.105509},
	journal = {Computers and Geosciences},
	author = {Song, Guofeng and Roubinet, Delphine and Wang, Xiaoguang and Li, Gensheng and Song, Xianzhi and Tartakovsky, Daniel M.},
	year = {2024},
	pages = {105509},
}

@Misc{amsmath,
  author =	 {{American Mathematical Society}},
  title =	 {User's Guide for the \texttt{amsmath} Package
                  (Version 2.0)},
  url =		 {ftp://ftp.ams.org/pub/tex/doc/amsmath/amsldoc.pdf},
  urldate =	 {2015-07-30},
  year =	 2002}

@article{aronszajn_theory_1950,
	title = {Theory of Reproducing Kernels},
	volume = {68},
	doi = {10.2307/1990404},
	number = {3},
	journal = {Transactions of the American Mathematical Society},
	author = {Aronszajn, N.},
	year = {1950},
	pages = {337--404},
}

@article{TRASK2022110969,
title = {Enforcing exact physics in scientific machine learning: A data-driven exterior calculus on graphs},
journal = {Journal of Computational Physics},
volume = {456},
pages = {110969},
year = {2022},
doi = {10.1016/j.jcp.2022.110969},
author = {Nathaniel Trask and Andy Huang and Xiaozhe Hu}}

@inproceedings{scholkopf_generalized_2001,
	address = {Berlin, Heidelberg},
	title = {A {Generalized} {Representer} {Theorem}},
	isbn = {978-3-540-44581-4},
	doi = {10.1007/3-540-44581-1_27},
	language = {en},
	booktitle = {Computational {Learning} {Theory}},
	publisher = {Springer},
	author = {Schölkopf, Bernhard and Herbrich, Ralf and Smola, Alex J.},
	editor = {Helmbold, David and Williamson, Bob},
	year = {2001},
	pages = {416--426},
}

\newpage
\appendix
\section{RKHS Spaces, Optimal GP Recovery, and the Representer Theorem}\label{app:RKHS}
For completeness, we repeat here the statement of \Cref{thm:inner_min} and provide its proof. Let $K: \mathcal{X} \times \mathcal{X} \rightarrow \R$ be a symmetric positive definite bivariate kernel and let $\mathcal{H}_K$ be the induced RKHS space with accompanying norm $\norm{\cdot}_K$. \Cref{thm:inner_min} is derived from the Representer Theorem and gives us a framework for solving the optimal recovery of a GP $f \in \mathcal{H}_K : \mathbf{X} \subset \mathcal{X} \mapsto \mathbf{Y} \subset \R$. 
\representer*
\begin{proof}
The minimization problem in \Cref{thm:inner_min} has a solution\footnote{In \eqref{eq:sol_rep_thm}, $K\left(\,\cdot\,\,, \mathbf{X} \right)$ represents the kernel function evaluated at the new test data and each of the training points in $\mathbf{X}$, and $V^*$ can be interpreted as a vector of coefficients weighting these training points. Thus, the function that minimizes the GP recovery objective lies in the span of the kernel evaluated at the training data points.} given by the Representer Theorem, that is
\begin{equation}\label{eq:sol_rep_thm}
f^*(\cdot) = K\left(\,\cdot\,\,, \mathbf{X} \right) V^*,
\end{equation}
where $V^*$ is the minimizer of  the problem
\begin{equation}
	V^* = \argmin_{V} \,\,  V^T K\left(\mathbf{X}, \mathbf{X} \right) V + \frac{1}{\sigma_{\epsilon}^2} \norm{K\left( \mathbf{X}, \mathbf{X} \right) V - \mathbf{Y}}_2^2.
\end{equation}

For simplicity, denote $\mathbf{K} = K(\mathbf{X},\mathbf{X})$.
Differentiating with respect to $V$, the minimizer $V^*$ solves
\begin{equation}
	2 \mathbf{K} V +\frac{2}{\sigma_{\epsilon}^2} \left( \mathbf{K}^T \mathbf{K} \mathbf{V} -  \mathbf{K} \mathbf{Y} \right) = 0,
\end{equation}
which has a solution
\begin{equation}
V^* = (\mathbf{K} + \sigma_{\epsilon}^2 I)^{-1} \mathbf{Y}.
\end{equation}
Thus,
$$f^*(\cdot) = K\left( \,\cdot\,\,, \mathbf{X} \right)  (\mathbf{K} + \sigma_{\epsilon}^2 I)^{-1} \mathbf{Y}.$$
Substituting this into the expression for $\ell(\mathbf{X};\mathbf{Y})$ yields
\begin{equation} \begin{split}
	\ell(\mathbf{X};\mathbf{Y}) &=  \norm{f^*}_K  + \frac{1}{\sigma_{\epsilon}^2} \norm{ f^*(\mathbf{X}) - \mathbf{Y} }_2^2\\
	&=  \mathbf{Y}^T (\mathbf{K}+\sigma_{\epsilon}^2 I)^{-1} \mathbf{K} (\mathbf{K}+\sigma_{\epsilon}^2 I)^{-1} \mathbf{Y} + \frac{1}{\sigma_{\epsilon}^2} \norm{ \mathbf{K} (\mathbf{K}+ \sigma_{\epsilon}^2 I)^{-1} \mathbf{Y} - \mathbf{Y} }_2^2 \\
&=  \mathbf{Y}^T (\mathbf{K}+\sigma_{\epsilon}^2 I)^{-1} \mathbf{K} (\mathbf{K}+\sigma_{\epsilon}^2 I)^{-1} \mathbf{Y}\\
	&\qquad\qquad + \frac{1}{\sigma_{\epsilon}^2} \mathbf{Y}^T (\mathbf{K}+\sigma_{\epsilon}^2 I)^{-1} \mathbf{K}\mathbf{K} (\mathbf{K}+\sigma_{\epsilon}^2 I)^{-1} \mathbf{Y} \\
	&\qquad\qquad - \frac{2}{\sigma_{\epsilon}^2} \mathbf{Y}^T \mathbf{K} (\mathbf{K}+\sigma_{\epsilon}^2 I)^{-1}\mathbf{Y}  + \frac{1}{\sigma_{\epsilon}^2} \mathbf{Y}^T \mathbf{Y}\\
& \stackrel{(i)}{=} \mathbf{Y}^T (\mathbf{K}+\sigma_{\epsilon}^2 I)^{-1} \mathbf{K} (\mathbf{K}+\sigma_{\epsilon}^2 I)^{-1} \mathbf{Y} \\
	&\qquad\qquad + \frac{1}{\sigma_{\epsilon}^2} \mathbf{Y}^T (\mathbf{K}+\sigma_{\epsilon}^2 I)^{-1} (\mathbf{K} + \sigma_{\epsilon}^2 I - \sigma_{\epsilon}^2 I) \mathbf{K} (\mathbf{K}+\sigma_{\epsilon}^2 I)^{-1} \mathbf{Y} \\
	&\qquad\qquad - \frac{2}{\sigma_{\epsilon}^2} \mathbf{Y}^T \mathbf{K} (\mathbf{K}+\sigma_{\epsilon}^2 I)^{-1} \mathbf{Y}  + \frac{1}{\sigma_{\epsilon}^2} \mathbf{Y}^T \mathbf{Y}\\
&= \mathbf{Y}^T (\mathbf{K}+\sigma_{\epsilon}^2 I)^{-1} \mathbf{K} (\mathbf{K}+\sigma_{\epsilon}^2 I)^{-1} \mathbf{Y} \\
	&\qquad\qquad + \frac{1}{\sigma_{\epsilon}^2} \mathbf{Y}^T \mathbf{K} (\mathbf{K}+\sigma_{\epsilon}^2 I)^{-1} \mathbf{Y} - \mathbf{Y} (\mathbf{K}+\sigma_{\epsilon}^2 I)^{-1} \mathbf{K} (\mathbf{K}+\sigma_{\epsilon}^2 I)^{-1} \mathbf{Y} \\
	&\qquad\qquad - \frac{2}{\sigma_{\epsilon}^2} \mathbf{Y}^T \mathbf{K} (\mathbf{K}+\sigma_{\epsilon}^2 I)^{-1} \mathbf{Y}  + \frac{1}{\sigma_{\epsilon}^2} \mathbf{Y}^T \mathbf{Y}\\
	&= -\frac{1}{\sigma_{\epsilon}^2} \mathbf{Y}^T \mathbf{K} (\mathbf{K}+\sigma_{\epsilon}^2 I)^{-1} \mathbf{Y} +\frac{1}{\sigma_{\epsilon}^2} \mathbf{Y}^T \mathbf{Y} \\
	& \stackrel{(ii)}{=} \frac{1}{\sigma_{\epsilon}^2} \left( - \mathbf{Y}^T (\mathbf{K} + \sigma_{\epsilon}^2 I - \sigma_{\epsilon}^2 I) (\mathbf{K}+\sigma_{\epsilon}^2 I)^{-1} \mathbf{Y} +  \mathbf{Y}^T \mathbf{Y} \right) \\
	&= \frac{1}{\sigma_{\epsilon}^2} \left( - \mathbf{Y}^T \mathbf{Y} + \sigma_{\epsilon}^2 \mathbf{Y}^T (\mathbf{K} + \sigma_{\epsilon}^2 I)^{-1} \mathbf{Y} +  \mathbf{Y}^T \mathbf{Y} \right) \\
&=  \mathbf{Y}^T (\mathbf{K} + \sigma_{\epsilon}^2 I)^{-1} \mathbf{Y}, \\
\end{split} \end{equation}
where (i) is obtained by adding and subtracting $\sigma_{\epsilon}^2 I$ in the second term, and (ii) is obtained by adding and subtracting $\sigma_{\epsilon}^2 I$ in the first term.
\end{proof}

We next prove an intermediate step for the proof of \Cref{lem:error}.
\begin{lemma}\label{lem:reproducing_property_example}
Let $\Omega\subseteq\mathbb{R}^d$, let $K:\Omega\times\Omega\rightarrow\mathbb{R}$ be a positive definite kernel with associated RKHS $\mathcal{H}_{K}$, and let $f$ be a true function from which we obtain data $(\mathbf{X},\mathbf{Y})=(x_i,y_i)_{i=1,...,N}$, where $f(x)=y.$ We denote by $K(x,\mathbf{X})$ a $1\times N$ vector with entries $K(x,x_i)$ for $i\in[N]$, and $K(\mathbf{X},\mathbf{X})$ an $N\times N$ matrix with entries $K(x_i,x_j)$ for $i,j\in[N]$. Then, by the reproducing property, we have:
\begin{align}
  \langle K(x,\mathbf{X})[K(\mathbf{X},\mathbf{X})+\sigma_{\epsilon}^2I]^{-1} K(\mathbf{X}, \cdot),K(x,\mathbf{X})[K(\mathbf{X},\mathbf{X})+\sigma_{\epsilon}^2I]^{-1} K(\mathbf{X}, \cdot)\rangle_{\mathcal{H}_{K}} &\leq K(x,\mathbf{X})[K(\mathbf{X},\mathbf{X})+\sigma_{\epsilon}^2I]^{-1} K(\mathbf{X}, x).
\end{align}
\end{lemma}
\begin{proof}
    Let $A=[K(\mathbf{X},\mathbf{X})+\sigma_{\epsilon}^2I]$. By the definition of the reproducing property, we have:
    \begin{equation}
    \begin{split}
          \langle K(x,\mathbf{X})A^{-1} K(\mathbf{X}, \cdot),K(x,\mathbf{X})A^{-1} K(\mathbf{X}, \cdot)\rangle_{\mathcal{H}_{K}} &= [A^{-1}K(\mathbf{X}, x)]^{\intercal} K(\mathbf{X},\mathbf{X})[A^{-1}K(\mathbf{X}, x)]\\
          &=K(\mathbf{X}, x)^{\intercal}A^{-1}K(\mathbf{X},\mathbf{X})A^{-1}K(\mathbf{X}, x).
    \end{split}
    \end{equation}
Then since $A=[K(\mathbf{X},\mathbf{X})+\sigma_{\epsilon}^2I]$, it is also true that:
$$K(\mathbf{X},\mathbf{X})=A-\sigma_{\epsilon}^2 I,$$ and therefore:
$$A^{-1}K(\mathbf{X},\mathbf{X})A^{-1}=A^{-1}-\sigma_{\epsilon}^2 A^{-2}.$$ Combining expressions, we obtain:
\begin{align}
     \langle K(x,\mathbf{X})A^{-1} K(\mathbf{X}, \cdot),K(x,\mathbf{X})A^{-1} K(\mathbf{X}, \cdot)\rangle_{\mathcal{H}_{K}} &= K(\mathbf{X}, x)^{\intercal}A^{-1}K(\mathbf{X}, x) - \sigma_{\epsilon}^2 K(\mathbf{X}, x)^{\intercal}A^{-2}K(\mathbf{X}, x),\\
     &\leq K(\mathbf{X}, x)^{\intercal}A^{-1}K(\mathbf{X}, x),
\end{align}
since the term $\sigma_{\epsilon}^2 K(\mathbf{X}, x)^{\intercal}A^{-2}K(\mathbf{X}, x)$ is nonnegative.
\end{proof}

\section{Block matrix formulation}\label{app:matrices}
In \eqref{eq:matrices}, we specify the matrices needed for the optimal solution of the unknown edge subproblem. We detail a subset of those matrices here. First, we have $F=\text{vec}(\mathbf{F}_\text{un})$, which is the row-major vectorization of $\mathbf{F}_\text{un}$:

\begin{equation}
    F = \begin{bmatrix}
        \mathbf{F}_{\text{un},(1,1)} &
        \hdots &
        \mathbf{F}_{\text{un},(1,N_\text{data})} &
        \mathbf{F}_{\text{un},(2,1)} &
        \hdots &
        \mathbf{F}_{\text{un},(2,N_\text{data})} &
        \hdots &
        \mathbf{F}_{\text{un},(E_\text{un},1)} &
        \hdots &
        \mathbf{F}_{\text{un},(E_\text{un},N_\text{data})} 
    \end{bmatrix}^\intercal.
\end{equation}
Next, we have $\delta_0,$ the incidence matrix or graph gradient matrix:

\begin{equation}
    \delta_0   =  \begin{bmatrix}
        \delta_{0 \mathcal{E}_\text{obs}, \mathcal{V}_\text{obs}} & \delta_{0 \mathcal{E}_\text{obs}, \mathcal{V}_\text{un}} \\
        \delta_{0 \mathcal{E}_\text{un}, \mathcal{V}_\text{obs}} & \delta_{0 \mathcal{E}_\text{un}, \mathcal{V}_\text{un}}
    \end{bmatrix}.
\end{equation}
We then take the Kronecker product of the bottom-right block, $D_0 = \delta_{0 \mathcal{E}_\text{un}, \mathcal{V}_\text{un}}$, with $I_{N_\text{data}}$ to obtain $\widehat{D_0}$:

\begin{equation}
    \widehat{D_0} = D_0 \otimes I_{N_\text{data}} = \begin{bmatrix}
        D_{0 (1,1)} I &  \hdots & D_{0 (1,V_{\text{un}})} I \\
        \vdots & \ddots & \vdots \\
        D_{0 (E_{\text{un}}, 1)} I & \hdots & D_{0, (E_{\text{un}}, V_{\text{un}})} I
    \end{bmatrix}     
\end{equation}

Finally, we form $\widehat{K},$ which is a block-diagnoal matrix of the Tikhonov-regularized kernel evaluations at each datapoint:

\begin{equation}
    \widehat{K} = \begin{bmatrix}
        \mathbf{K}_1^{-1} & 0 & \hdots & 0 \\
        0 &  
        \mathbf{K}_2^{-1} & \ddots & \vdots \\
        \vdots & \ddots & \ddots & 0 \\ 
        0 & \hdots & 0 &  
        \mathbf{K}_{E_\text{un}}^{-1}
    \end{bmatrix},
\end{equation}
with 

 \begin{equation}
     \mathbf{K}_i = \begin{bmatrix}
        K_{i}(\mathbf{u}_{i,1},  \mathbf{u}_{i,1})+\sigma_{\epsilon}^2 & K_i(\mathbf{u}_{i,1},  \mathbf{u}_{i,2}) & \hdots & K_i(\mathbf{u}_{i,1},  \mathbf{u}_{i,N_\text{data}}) \\
        K_i(\mathbf{u}_{i,2},  \mathbf{u}_{i,1}) &  
        K_{i}(\mathbf{u}_{i,2},  \mathbf{u}_{i,2})+\sigma_{\epsilon}^2 & \ddots & \vdots \\
        \vdots & \ddots & \ddots & \vdots \\ 
        K_i(\mathbf{u}_{i,N_\text{data}},  \mathbf{u}_{i,1}) & \hdots & \hdots &  
        K_{i}( \mathbf{u}_{i, N_\text{data}},  \mathbf{u}_{i,N_\text{data}})+\sigma_{\epsilon}^2
    \end{bmatrix},
\end{equation}
for $i\in\mathcal{E}_{\text{un}}$, where $K_i$ is the kernel learned for edge $i$\footnote{Recall that we learn a lengthscale parameter for each edge, thus defining a unique kernel on each edge of the graph.} and $\mathbf{u}_{i,j}$ denotes the $j$th datapoint for the endpoints of edge $i$.

\newpage
\section{Training details}\label{app:training}
In this section, we provide results pertaining to the training process. These results were generated with the discrete fracture network dataset. \Cref{fig:training_time_per_epoch} illustrates that per-epoch runtime grows approximately quadratically in the number of training samples. This is consistent with what we would expect given the dimensionality of the problem.

\begin{figure}[ht!]
\centering
\includegraphics[width=0.4\linewidth]{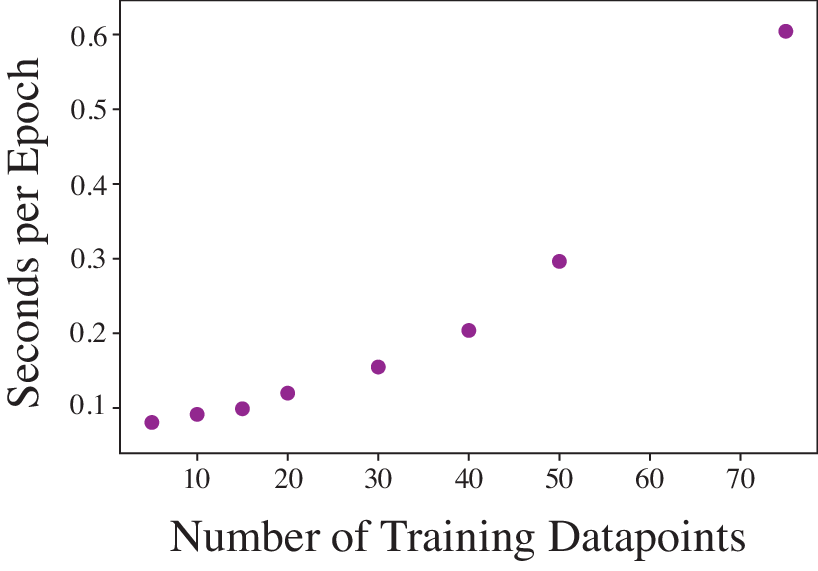}
\caption{Scaling of training time per epoch with number of training samples.}
\label{fig:training_time_per_epoch}
\end{figure}

\end{document}